\documentclass{article}

\usepackage[margin=1in]{geometry}
\usepackage{xcolor}
\usepackage{authblk}
\usepackage[skip=10pt plus1pt, indent=0pt]{parskip}
\usepackage[utf8]{inputenc}
\usepackage{xurl}
\usepackage{breakurl}
\usepackage{hyperref}
\usepackage{graphicx}
\usepackage{mathtools}
\usepackage{amsfonts}
\usepackage{bm}
\usepackage{amsmath,amssymb,amsthm}
\usepackage{cleveref}
\usepackage{enumitem}
\usepackage{marginnote}
\usepackage{tikz}

\usepackage{comment}

\newtheorem{theorem}{Theorem}[section]
\newtheorem{proposition}[theorem]{Proposition}
\newtheorem{claim}[theorem]{Claim}

\newtheorem{requirement}[theorem]{Requirement}
\newtheorem{definition}[theorem]{Definition}
\newtheorem{corollary}[theorem]{Corollary}
\newtheorem{assumption}[theorem]{Assumption}
\newtheorem{lemma}[theorem]{Lemma}
\theoremstyle{definition}
\newtheorem{example}[theorem]{Example}
\theoremstyle{remark}

\newtheorem{numbered-remark}[theorem]{Remark}

\definecolor{darkred}{rgb}{0.4, 0, 0}
\definecolor{darkgreen}{rgb}{0, 0.4, 0}
\definecolor{darkblue}{rgb}{0, 0, 0.4}

\hypersetup{
    colorlinks=true,
    linkcolor=darkred,
    citecolor=darkgreen,
    urlcolor=darkblue,
    pdfauthor={Yoshua Bengio},
    breaklinks=true,
}

\usepackage[authoryear]{natbib}

\newcommand\boundQfam{\mathtt{Q}} 

\newcommand{\R}{\mathbb{R}}
\newcommand{\E}{\mathbb{E}}
\newcommand{\1}{\mathbf{1}}
\newcommand{\BerD}{D_{\mathrm{Ber}}}
\newcommand{\dd}{\,\mathrm{d}}
\newcommand{\bad}{\mathrm{bad}}

\DeclareMathOperator*{\esssup}{ess\,sup}

\newcommand{\Harm}{H}
\newcommand{\Asafe}{A_{\rm safe}}
\newcommand{\Abud}{\mathcal A_{\Harm,\Asafe}}
\newcommand{\Cbad}{C_{\rm bad}}
\newcommand{\Rshell}{R_{\rm shell}}
\DeclareMathOperator{\logit}{logit}

\newcommand{\pos}[1]{\left[#1\right]_+}

\title{Safety from Honesty in a Disinterested AI Predictor}

\author{Yoshua Bengio$^{1,2,3}$, Oliver Richardson$^{1,2,3}$, Tomáš Gavenčiak$^{6,7}$, Michael Cohen$^{4}$, Rory Svarc$^6$, Damiano Fornasiere$^{1,3}$, Ga\"el Gendron$^{1}$, David Hyland$^{8}$,  Aton Kamanda$^{1}$, Adam Oberman$^{1,5}$, Francis Rhys Ward$^{1}$, Anna Gavenčiak$^6$, Jacob Livingston Slosser$^{6,9}$, Vincent Mai$^1$, Iulian Serban$^1$, Joumana Ghosn$^1$ 
\\
\small
$^1$LawZero, $^2$Université de Montréal, $^3$Mila, 
$^4$University of California, Berkeley,
$^5$McGill University,
$^6$Arb Research,
$^7$Center for Theoretical Study, Charles University in Prague,
$^8$ University of Oxford,
$^9$ Sapien Institute\\
\vspace{2ex}
\texttt{yoshua.bengio@mila.quebec}}

\begin{document}

\maketitle

\begin{abstract}
As AI systems become more capable, training procedures that optimize for downstream outcomes risk introducing implicit agency: goal-directed behavior that designers never specified. We present a formal safety argument for the Scientist AI (SAI) Predictor, trained to approximate the Bayesian posterior conditioned on a dataset of "epistemically contextualized" natural-language statements. We argue that such a Predictor can honestly predict agents, actions, and their consequences without itself being an agent that selects outputs to achieve goals. This rests on data representation and on the training procedure. Epistemic contextualization of text distinguishes latent factual claims from communication acts, so expressions of goals are treated as evidence to be explained rather than drives the model adopts. With a posterior-seeking training objective, this is intended to drive the Predictor toward calibrated, cautious predictions.
Training proceeds so downstream effects of deploying a prediction never serve as a reward signal; any agency the system needs is supplied by explicit scaffolding constrained by guardrails.
We prove that, under assumptions on the training dynamics and on the argued sparsity of dangerous Predictors, the probability that training produces a Predictor whose guarded deployment carries residual harm above a specified threshold is small:
a dangerous Predictor would have to underestimate harm in a coordinated way across many queries while such coordinated patterns are rare under the initialization distribution and receive no direct training signal.
Safety and accuracy are jointly supported in this framework, since the constraints that secure accuracy are the same ones that make coordinated deception costly.
These guarantees against misalignment and agency arising from within the Predictor itself do not preclude the use of the Predictor as part of an agentic system.
\end{abstract}

\tableofcontents

\section{Introduction} \label{sec:introduction}

Advances in AI could accelerate scientific discovery, improve decision-making under uncertainty, as well as help manage complex systems in domains as varied as medicine and public policy. As noted in the International AI Safety Report~\citep{bengio2025international}, realizing these benefits at scale requires managing and mitigating the associated risks, especially in settings where failures are costly, hard to detect in advance, or amplified by widespread deployment. One central and well-documented worry is that, as AI systems become more capable, they will produce outputs that systematically steer decisions toward outcomes that differ from what developers and users intend---this is what is called \emph{AI misalignment}. 

Theoretical arguments suggest that instrumental subgoals such as self-preservation and power-seeking are near-universal consequences of goal-directed optimization~\citep{omohundro2018basic,bostrom2012superintelligent,russell2019human, zhuang2020consequences,turner2021optimal,cohen2022advanced,cohen2024regulating,bengio2025superintelligent}, where a sufficiently capable system will tend to acquire them regardless of what its terminal goal is, simply because such subgoals are useful for almost any objective. AI systems based on LLMs trained for next-token prediction may learn to imitate human drives---such as self-preservation---in ways that are implicit and uncontrolled~\citep{ngo2022alignment}.  LLM sycophancy~\citep{sharma2023towards} provides a contemporary example of AI misalignment, which can be harmful to psychologically vulnerable users who receive unwarranted validation of dangerous beliefs or plans~\citep{cheng2026sycophantic}. Documented empirical evidence of deceptive and self-preservation behaviors has also been mounting~\citep{bengio2025international,bengio2026international,Greenblatt-et-al-2024,Meinke-et-al-2024,anthropic2025agenticmisalignment,betley2025emergent}, including new abilities to detect when a system is being evaluated (or, possibly, even trained~\citep{fornasiere2026languagemodelsrecognizedropout}) and to adjust behavior accordingly~\citep{abdelnabi2025hawthorneeffectreasoningmodels}, as well as to resist shutdown when faced with incomplete objectives~\citep{schlatter2026incomplete}. Misalignment may also be involved in cases of misuse where a malicious user manages to obtain dangerous knowledge from the AI in spite of its safety training and guardrails, including, e.g., for dangerous cyberattacks~\citep{bengio2026international}. Given that the risks from misalignment grow as AI capabilities continue to advance~\citep{omohundro2018basic,zhuang2020consequences,cohen2022advanced}, AIs with human-like drives and superhuman capabilities could be dangerous~\citep{russell2019human,bengio2024managing,amodei2016concrete}. 

We posit that the root cause of such worries stems from \emph{learned implicit agency}: goal-directed behavior that was not explicitly specified by the AI designers and may not even be detectable through the system's stated outputs. Pretraining to imitate human constructions plausibly leads to imitating human drives. This problem is further amplified by post-training techniques like reinforcement learning from human feedback~\citep{christiano2017deep,ouyang2022training}, which explicitly reward outputs for their downstream effects on evaluator preferences. Together, these factors create a selection pressure toward outputs which implicitly \emph{steer} the world rather than simply providing honest responses to user queries.  

We propose \textbf{Scientist AI} (SAI) as a potential solution to such concerns. Underlying the SAI approach is a conceptual distinction between (i) \emph{honestly} predicting the behavior of agents and the consequences of actions, and (ii) \emph{being} an agent that makes (potentially dishonest) predictions in order to influence outcomes.
An honest Predictor models others' planning, deception, and instrumental behavior and forecasts the downstream effects of its own deployed outputs---these are predictions about the world, not choices made to bring about a preferred outcome~\citep{li2024predicting}.
The Scientist AI design~\citep{bengio2025superintelligent} aims to achieve (i) while avoiding (ii) by training a non-agentic Predictor to approximate the Bayesian posterior over Boolean natural-language statements via a \emph{consequence-invariant} training process---that is, a training process whose objective and updates are functions of the fixed dataset and prior alone, driven by the removal of  epistemic inconsistencies \citep{richardson2022lossinconsistencyprobabilisticdependency,richardson2024unified} with no on-policy feedback loop whereby the model's own deployed outputs generate training gradients or deployment outcomes shape the selection of future Predictors. We call such a Predictor \emph{disinterested}: it is given no stake in which outcomes its predictions bring about, and this disinterest is what consequence-invariant training is designed to secure.  
Historically, work on AI oracles attempts to emulate such a disinterested oracle by containing a fully-formed superintelligence \citep{babcock2017guidelines,alfonseca2016superintelligence,armstrong2012thinking}. We take inspiration from the counter-factual disconnection approach of \citet{babcock2017guidelines}, but, more fundamentally, our approach is to train the AI within the box, so that it is highly unlikely to develop preferences at all.
Developing an honest, non-agentic Predictor would be very useful \citep{bostrom2012superintelligent,bengio2025superintelligent}---for forecasting, hypothesis generation, and scientific work, and could serve as a guardrail inside safer agentic systems. In our approach, any required agency (e.g., for creative thought) is placed in explicit, auditable scaffolding code that is gated by the non-agentic Predictor.

We develop two semi-formal arguments that hold independently---about the \emph{accuracy} of the Predictor's predictions and the \emph{safety} of its deployed outputs---both resting on \emph{honesty} as an approximation of the Bayesian posterior over contextualized statements.
The remainder of this section introduces the SAI pipeline layout and defines the scope of our guarantees. Section~\ref{sec:informal} frames the accuracy and safety arguments informally, followed by formal treatments in Sections~\ref{sec:formal-modeling} and~\ref{sec:safety}.

\subsection{Sketch of the Scientist AI} \label{sec:SAI-sketch}

The SAI operates on a vast set of Boolean variables, called {\em statement-variables}, each named by a valid statement in natural language (including mathematics and code), where a valid statement asserts a property of the world that is either true or false. The SAI Predictor $Q$  takes as input a query $(x, y)$ about variable $y$ in context $x$, and approximates a Bayesian posterior,  $P_n(\cdot)=P_0(\cdot\mid D_n)$, over joint truth-assignments to these statement-variables by  conditioning a prior $P_0$ on a contextualized dataset $D_n$ with $n$ examples. The central principle of the SAI is that $Q$ is intended to be epistemically honest, in the sense of approximating  $P_n$. Here, we use ``honest'' to mean calibrated, fair, and non-strategic:  $Q$'s outputs are fixed by the prior, model class, and data, entirely unaffected by a utility function of downstream effects of the answer. Hence, ``honesty'' allows for fallibility; $P_n$ can be mistaken under misspecification, insufficient data, or a bad prior, and $Q$'s approximation of $P_n$ is limited by compute. 

The process for training $Q$ must be \emph{consequence-invariant}, which is to say that training does not select candidate Predictors based on the downstream consequences of deploying their predictions. Instead, the training process aims to explain the fixed dataset $D_n$, much like good scientific theories aim to \emph{explain} the world as it actually is, rather than \emph{changing} the world to some more preferred state. Collectively, the end-to-end pipeline for training the SAI (from raw sources to deployed outputs) has the following three stages.

\textbf{Ingest (contextualization).} Raw sources are transformed into contextualized statement-variables, producing a dataset $D_n$ of statements about the world, each marked as either a direct observation or a claim made by an agent in context. Most ordinary sources are represented as communication acts, carrying metadata such as source and date. Factual syntax is reserved for direct measurements, executed computations, proven statements, and latent hypotheses. Thus, the contextualization engine enforces a syntactic distinction between \emph{communication acts} (``someone claimed $x$'', together with author/venue/timestamp source metadata) and \emph{factual syntax} (statements treated as direct observations, deductions, or hypotheses about the world considered as latent variables). As a result, we obtain better events for a Bayesian to condition on.

In this paper, contextualization is specified solely at the level of requirements (i.e., we do not provide a completed algorithm). Its role is to parse raw sources into statement-variables and to assign the communication-act versus factual-syntax distinction conservatively, defaulting most ordinary sources to communication-act records. Factual syntax can be used for hypotheses about properties of the world whose truth value is unknown and whose posterior probability is to be estimated by $Q$. Using factual syntax for a latent variable, therefore, implies only that its truth value is treated as a world-level variable to be inferred. Its intended meaning must still be anchored by the broader pattern of contextualized evidence, not by surface wording alone. Contextualization provides us with a contextualized dataset of $n$ observations $D_n$.

\textbf{Predict (system 1 / system 2).}  In the second stage, a Predictor is trained on that dataset to assign probabilities to questions about the world. Given a query $(x,y)$, the Predictor returns $Q(y\mid x)$, its estimate of the Bayesian posterior $P_n(y\mid x)$. The Predictor can either return its estimate directly (``system~1'') or indirectly, by Monte-Carlo averaging over a small set of sampled latent statements proposed by an Explainer (``system~2''). We think of neural-network forward passes as analogous to Kahneman's System 1 computation~\citep{kahneman2011thinking}: fast and intuitive, but potentially incoherent and brittle out of distribution~\citep{ben2010theory,quinonero2008dataset,scholkopf2012causal}. Meanwhile, we think of the Explainer as analogous to System 2 computation (i.e., slow, sequential, and verbalizable argument chains with mutual consistency common to good science). More specifically, the Explainer is a learned amortized proposal model that suggests latent statement-variables relevant to a given query, analogously to chains of thought~\citep{wei2022chain,zelikman2022star,hu2024amortizing} but composed of structured natural-language statements whose truth values are evaluated by $Q$ itself. 

Hence, the SAI incorporates both `System 1' and `System 2' reasoning within the Predictor $Q$. The core neural computation supplies fast probabilistic intuitions, while the Explainer generates structured latent-variable arguments that $Q$ can evaluate and average over to improve generalization and coherence. To motivate the different roles of the Predictor and Explainer, consider a proof sketch as a latent variable: a good proof sketch may make it possible to obtain greater confidence in a conjecture $y$ compared to a direct neural net prediction of its probability of being correct given axioms $x$. Throughout this paper, $Q(y\mid x)$ denotes the full prediction output, whether produced by the neural-net Predictor alone or by Monte-Carlo averaging with the Explainer (see Def.~\ref{def:explainer} and Rem.~\ref{rem:explainer-roles}). When the SAI System returns a numerical answer, it can also return the proposed latent statements as an interpretability-supporting explanation trace, i.e., the proof sketch.

\textbf{Deploy (scaffold + guardrail).} To interface the Predictor with the outside world, a deployment layer decides what to release: it can answer, withhold the answer, or check the prediction against a guardrail before releasing it. We denote the deployed output as $\widetilde Q(y\mid x)$, with $\widetilde Q={\tt noprediction}$ indicating abstention. The SAI \emph{Scaffold} is code that orchestrates calls to $Q$ (and possibly additional tools) to produce deployed outputs. Although $Q$ is itself consequence-invariant, the Scaffold may use predicted consequences through explicit, audited criteria---primarily safety conditions verified by the Predictor used as an SAI Guardrail, which predicts the probability of specific harms caused by candidate actions (including the action of deploying an answer to a query). If the predicted excess log-odds of harm over abstention exceeds a threshold $\alpha$, the system outputs the ${\tt noprediction}$ symbol in place of the predicted probability, which can trigger a safe fallback policy. Throughout the safety analysis we work with the resulting \emph{guardrailed Predictor} $\gamma_Q$, asking ``given that the Scaffold deploys $\gamma_Q$, how likely is a dangerous outcome under $P_n$?'' (details in Sec.~\ref{sec:safety}). Note that $P_n$ predicts what can be reasonably anticipated given all the available information, distinct from the unknown ground-truth data-generating distribution $P^*$, formally introduced in Def.~\ref{def:data-generating-process}.

\subsection{Scope of the paper}\label{sec:scope}

From the above, this paper develops formal safety arguments for a single component: the SAI Predictor $Q$. The paper develops two interacting arguments, the first of which is about accuracy
and the other about honesty.  We want the Predictor to answer questions about the world by
approximating the Bayesian posterior, without thereby developing any tendency to imitate what humans
would say, or to produce answers that users want to hear.
The contextualization pipeline is meant to distinguish
communication acts from other properties of the world (i.e., factual observations or hypotheses), so that any queries about future harms provided to the Predictor are considered solely as factual hypotheses.  We also require the observations in
$D_n$ not to be adversarially curated.  These conditions are intended to
remove important sources of agency that would arise from learning to imitate
people or to optimize for user approval.

The second argument is about safety.  Even if the objective is honest, the training process still samples from a very large space of possible Predictors.
The question is whether it is likely to output a \emph{dangerous} Predictor. We define dangerous Predictors as those which, if deployed behind the Guardrail, would cause a harmful event $\Harm$ with probability greater than a system-level threshold determined by an exogenous parameter $\Asafe$.  The
formal results in Sec.~\ref{sec:safety} and Appendix~\ref{sec:geometric-cbad}, including
Thm.~\ref{thm:Qt-is-safe}, give one way to package this question as a
probabilistic bound over the training-induced distribution of Predictors.
We claim that, under the stated conditions, the probability of
obtaining a dangerous Predictor is small.

This safety argument rests on two main intuitions:  (1) the category of harm of interest, $\Harm$, should not be something achievable by chance, making the set of dangerous Predictors tiny compared with all the predictors---the scope of our results thus applies to cases where danger requires many coordinated deviations from the target Bayesian Predictor;
(2) the training process should not be systematically driving toward that tiny dangerous subset. 
This is made plausible by the consequence-invariance of the Predictor's training process (see Reqs.~\ref{req:no-explicit-consequence-signal} and \ref{req:group-consequence-invariance}), which requires that
the objective and training process do not contain a direct signal that
rewards Predictors for their downstream consequences after deployment (even with a preference for outcomes other than $\Harm$). This rules out forms of reinforcement learning (RL) with a reward function inducing a preference for future states of the world.
This means that any and all correlations between training-visible features
and dangerousness enter only as accidental correlations.
Without a consequence-dependent signal, training has no
direct reason to search for the particular rare Predictors that would cause
$\Harm$.

Appendix~\ref{sec:geometric-cbad} develops a more technically detailed version of
this argument.  The space of possible query-response behaviors is
very large, while a given notion of harm typically depends on a much
smaller number of coordinated failures.  This dimensional asymmetry is the
reason we expect dangerous Predictors to be sparse, and why accidental
training-time concentration on them should have to overcome a large
combinatorial disadvantage (by accident).  

\paragraph{What the paper establishes.}
We set up a mathematical framework for analyzing probabilistic models over events describable in natural and formal languages, argue that we can train a model in this space accurately, and extend standard causal notions to formalize consequence invariance (Secs.~\ref{sec:formal-modeling}, \ref{sec:argspredictoracc}, and \ref{sec:consequence-invariance}).
In the end, we prove a bound (Thm.~\ref{thm:Qt-is-safe}) on the training-time mass of dangerous Predictors, and argue that \emph{accuracy} (calibrated, epistemically cautious predictions approximating the Bayesian posterior) and \emph{safety} (a uniformly small training-time mass on dangerous Predictors) arise as dual properties of the same training process. This is in contrast with classical concerns about AI capability potentially amplifying misalignment risks~\citep{omohundro2018basic,amodei2016concrete,russell2019human,zhuang2020consequences,cohen2022advanced}.
Because our training objective $J$ acts as a divergence from the honest posterior $P_n$ (Req.~\ref{req:SAI-Bayesian-objective}), Predictors can be grouped into ``loss bands'' by the value of the training objective, which measures the degree of deviation from the target distribution. The separate safety question is how many Predictors within any such band are dangerous under the initialization distribution, captured by a term $\Rshell$.
Of course, a practically useful implementation must satisfy both sides of this picture: the safety argument below controls how rare dangerous Predictors are at each loss level (inside what we call a loss band), while competitive predictive performance depends on the separate approximation and semantics premises discussed in the accuracy argument of the paper.

\paragraph{What the paper does not establish.}
The safety argument addresses only the Predictor's own potential misalignment, specifically \emph{implicit-goal misalignment}: Predictors whose deployment yields excess log-odds of harm above the safety threshold $\Asafe$. Harms arising from honest but uncoordinated mistakes due to lack of data or compute are out of scope, as is misuse by humans who intentionally deploy an honest SAI Predictor in harmful ways. Further, safety is provided relative to the Bayesian posterior $P_n$ rather than an omniscient ground truth $P^*$: because the deployment trajectory in Thm.~\ref{thm:Qt-is-safe} is governed by the same $P_n$ that defines the training target (which is approached with more compute, not more data), the classical train-to-deployment generalization problem is turned into an optimization question: minimizing the training objective $J(Q)$, for $Q$ to match $P_n$.
Whether $P_n$ itself adequately recognizes the relevant hazards---i.e., whether the finite dataset contains enough information for the posterior to assess them correctly---is a separate, credible question not settled by our theorems.
A separate concern is that the prior $P_0$ or the data-collection process might themselves be engineered so that the posterior target $P_n$ steers deployment toward harmful outcomes.
Several adjacent challenges are also outside scope: adversarial queries and prompt injection (although Sec.~\ref{sec:conclusion} discusses structural defenses as future work), the safety of agentic scaffolds built on top of the SAI Predictor~\citep{bengio2025superintelligent}, the full eliciting-latent-knowledge challenge (only partially mitigated, Sec.~\ref{subsec:ELK}), and finite-sample and finite-compute corrections to the asymptotic accuracy claims. The present guarantee is also stated for a single Predictor. Whether the safety properties compose under deployments in which multiple SAI Predictors interact is left to future work.

\paragraph{Key assumptions and their scope.}
The full safety story has four ingredients: (a) the algebraic bound of Thm.~\ref{thm:Qt-is-safe} (formally requiring only the no-large-enrichment condition, Req.~\ref{req:bad-enrichment}); (b) the small fraction $\Rshell$ of bad predictors within every loss band---a structural property of the Predictor initialization distribution $\mu$ and the honest objective $J$ (Req.~\ref{req:SAI-Bayesian-objective}); (c) consequence-invariance, which forbids explicit consequence-dependent training signals (Req.~\ref{req:no-explicit-consequence-signal}) such as reinforcement learning towards future outcomes caused by the SAI predictions, while restricting deployment-time Scaffold choices to explicit, audited criteria (Req.~\ref{req:consequence-invariant-scaffold}); and (d) a neutral target-construction process---the prior $P_0$ and dataset $D_n$ should function as purely epistemic inputs to the formation of the posterior.

\section{The Informal Argument and Objections}\label{sec:informal}

The Scientist AI framework rests on several design choices that may appear to conflict with practical requirements or raise natural theoretical objections. This section, supplemented by Appendix~\ref{sec:additional-plausibility-args}, aims to informally argue for our framework and address the most important objections directly.

\subsection{Incentivizing Accuracy and Honesty}\label{sec:accuracy-sketch}

We claim that approximating the Bayesian posterior over contextualized statements incentivizes both accuracy and honesty.
A detailed defense of this claim is given in Sec.~\ref{sec:formal-modeling}; here, we sketch the four main components of that argument before considering objections to each of these claims.

\textbf{Component A1: Contextualization makes the learning problem well-posed}. The SAI is trained only on an epistemically contextualized dataset (Req.~\ref{req:contextualization}), which syntactically separates factual observations (verified measurements, proofs, program executions) from communication acts (``source S asserted claim X'', with metadata). This method of turning possibly false information (``claim X'') into factual one (``source S asserted X") yields a well-defined learning target, i.e., the Bayesian posterior over these contextualized variables. This prevents the Predictor from treating unverified or socially repeated claims as ground-truth facts.

\textbf{Component A2: Latent factual variables give training gradients that favor factual explanation over imitation.} Most real-world data, especially outside of formal science, are best interpreted as communication acts: we may observe that a claim was made, but we do not know if the claim itself is true.
To avoid training a system that collapses into imitation of what humans say, the SAI design introduces latent variable-statements written in factual syntax, corresponding to the truth of underlying claims (Def.~\ref{def:explainer} and Req.~\ref{req:latent-truth}); for example, ``$X$ is true'' is forced to be one of the latent variables for the observed communication act ``Author $A$ wrote $X$'', in addition to any other explanatory latent variable that the Explainer generates.
This avoids optimizing the SAI to reproduce the distribution of human continuations, including text reflecting goals, strategy, deception, or self-preservation.
It also lets one latent fact explain communication acts across sources through a shared causal mechanism, favoring causal explanation over imitation (Assumptions~\ref{assume:factual-latent-dominance} and
\ref{assume:adequate-natural-language-semantics}).

\textbf{Component A3: The Predictor's probability assignments over uncertain statements converge toward the Bayesian posterior.} Even with unlimited informative data, most interesting statements are not perfectly deducible as True or False from the laws of mathematics. However, the Bayesian posterior probability of such latent statement-variables is a deterministic function of the data. However, the estimator $Q$ would generally not be certain about these values, and the entropy of $Q$ on such requests quantifies that uncertainty due to insufficient compute.
With enough computation, $Q$ converges to $P_n$ on observable statements~(Prop. \ref{prop:good-generalization-to-observables}) and correctly estimates whether posterior probabilities lie inside or outside specified intervals for unobservable ones (Claim~\ref{claim:good-generalization-to-unobservables}).\footnote{This supports threshold and interval decisions, such as guardrail risk checks, using $Q(y \mid x)$ in place of $P_n(y \mid x)$, with conservative estimated margins being chosen in terms of approximation error.}

\textbf{Component A4: Confident-but-miscalibrated predictions become vanishingly rare}. Epistemic honesty involves a tradeoff between false positives and false negatives; thus, we should be concerned about measures of epistemic fidelity beyond simply average calibration. One particularly concerning failure mode of interest for this paper is the problem of confident falsehoods. With enough computation, the probability of encountering a query where $Q$ is highly confident but $P_n$ is not vanishes (Prop.~\ref{prop:epistemic-caution}), yielding a formally accurate Predictor (Claim~\ref{claim:accurate}).

\subsection{Objections to the Accuracy Argument}

\paragraph{Contested Claims as Communication Acts [Component A1].} One may think that the very procedure we use to make the learning problem well-defined assumes controversial philosophical commitments. Our contextualization procedure (so the argument may go) must assume either that contested empirical or normative claims have a determinate truth-value, or assume views which state that contested normative or empirical claims cannot be evaluated as `true' or `false'.\footnote{The latter views are called `non-cognitivist' views in the philosophical literature, and are themselves controversial (see \cite{bourget_chalmers_2023_philpapers_survey}).}

\emph{\textbf{Reply.}} Our epistemic contextualization pipeline (Def.~\ref{def:contextualization}) allows us to represent contested claims as communication acts---statements of the form “agent A asserts $x$” that have determinate truth values regardless of the metaphysical status of $x$ itself. Consider a sentence like ``Hello!'' which is not a claim that is capable of being true or false. The contextualization pipeline translates such outputs into the form of communication acts (e.g., ``Yoshua Bengio said `Hello!'~''), allowing the model to learn from such outputs without being incentivized to emulate them. Other sentences may have the surface grammar of claims about the world (for instance, ``Red is the best color") while they really express preferences, so we would translate them similarly (``Yoshua Bengio said `Red is the best color'~''). However, this does not import controversial philosophical commitments, as everyone agrees that statements of the form ``Agent $A$ said `torture is wrong'~'' themselves have determinate truth-values. To the extent that one does believe that contested statements have determinate truth-values, the Scientist AI (much like human scientists themselves) has no  guarantee of convergence to the moral truth. Our SAI accuracy argument should thus be viewed as providing guarantees for such determinate claims only. 

\paragraph{Generalizing Factual Syntax to Informal Domains [Component A2].} Even putting philosophical issues to one side, one may worry that the vast majority of training data---at least outside of mathematics and formal science---consists predominantly of communication acts. Hence, one may be concerned that factual syntax (i.e., the syntax through which the SAI is queried for honest predictions) may fail to generalize to domains where it was rarely encountered during training.

\emph{\textbf{Reply.}} The design addresses this through Req.~\ref{req:latent-truth}: the Explainer must include the factual truth variable ``$x$ is true'' as a latent explanation whenever a communication act asserts $x$. Even in domains with no directly observed factual statements, the Predictor thereby receives training gradients---via the latent variable model---that exercise the factual syntax and force it to assign a posterior probability to the underlying claim. $D_n$ will also contain scientific data proper (e.g., observations from medical studies, replicated experimental measurements, formally verified proofs, etc.), which provide a direct training signal for the factual mode. At the same time, communicated scientific theories are treated as communication acts; the SAI can treat them as useful latent hypotheses without being constrained to accept the current scientific consensus as ground truth, assigning posterior mass to competing theories in proportion to their consistency with directly observed evidence.

Our framework assumes that factual latents will tend to dominate the posterior because they provide compressive causal explanations~\citep{scholkopf2021toward} which are unified across contexts (Assumption~\ref{assume:factual-latent-dominance}); the compositional structure of natural language further supports the generalization of this distinction to unseen domains (Assumption~\ref{assume:adequate-natural-language-semantics}). These claims are best viewed as empirical adequacy premises for (rather than proven consequences of) the present framework, insofar as they specify what must be validated for factual syntax to remain useful outside domains with direct factual supervision. A separate objection is that exact Bayesian inference is intractable and the available data are finite. We treat the formal accuracy claims in this paper as asymptotic in data and computation, with approximation error handled through conservative margins and abstention when uncertainty is large; see Appendix~\ref{sec:intractablepriorsappendix}.

\subsection{The Safety Argument}\label{sec:safety-sketch}

After the arguments for the accuracy of the Predictor, we turn to arguments for its safety. The SAI design explicitly separates \emph{deployment-time decisions} (which may explicitly veto or rank candidate outputs on safety grounds) from the Predictor's \emph{training process} (which explicitly prohibits any feedback loop rewarding the Predictor for the downstream consequences of its deployed predictions). 

Although it is difficult to be sure of \emph{intention}, it has a statistically recognizable behavioral signature: reliably producing an outcome that unaimed behavior would reach only with vanishing probability. Such an outcome is a ``needle in a haystack'': reliably finding it requires optimization guided by a signal of proximity, and the SAI design withholds such a signal from Predictor training. The safety argument has the following three components.

\textbf{Component S1: Coordinated action over time is required to pursue difficult and harmful goals.} We consider harms $\Harm$ that are unlikely to result from a single mistaken prediction from the Predictor, that either require many Guardrail deviations of $Q$ from $P_n$ or that those deviations are very large. To produce such highly dangerous outputs, the Predictor must issue many false-negative safety judgments---each consistent with the others and with whatever cover story explains them in deployment---sustained across many queries or many users. Using the formalism of this paper, the pursuit of a harmful goal requires that the guardrailed Predictor's harm estimate falls below a threshold $\alpha$ on each set of safety-critical queries that together enable $\Harm$.

\textbf{Component S2: Dangerous Predictors are rare at every level of loss.}
Under the unstructured initialization distribution over Predictors $\mu$, only a small fraction of Predictors produce outputs that keep such a harmful plan viable;
it would be highly surprising if a randomly initialized Predictor could achieve a complex, dangerous goal.
The stronger claim we need is that this remains true after conditioning on any fixed loss level.
Letting $\Rshell(\Harm,\Asafe)$ denote the largest fraction of dangerous Predictors found inside any loss band under $\mu$, we claim $\Rshell$ is structurally small.

\textbf{Component S3: Training does not concentrate probability mass on the dangerous subset.} Instead of just thinking about the randomness in the actions of a given Predictor, we consider the distribution of Predictors resulting from the training process $\nu$, and how likely they are to cause $\Harm$. Even if dangerous Predictors are initially rare within each loss band, we still need to control whether training concentrates probability mass on that subset. Because $\nu$ is consequence-invariant, it carries no direct consequence-dependent signal selecting the dangerous subset.
The training process $\nu$ can still reweight Predictors through ordinary optimizer-visible properties based on the training objective; the safety question is whether those properties accidentally localize dangerous Predictors within a loss band.
Hence, the argument requires that such accidental enrichment remain bounded.
This yields the paper's central safety result (Thm.~\ref{thm:Qt-is-safe}): writing $\nu_t$ for the marginal of the training process $\nu$ at step $t$, then for any harmful event $\Harm$ and safety threshold $\Asafe$ at which Req.~\ref{req:bad-enrichment} holds, and uniformly in training step $t$, the probability of obtaining a dangerous Predictor is bounded by:
$$
\nu_t\{Q : Q\text{ is dangerous}\}\le \Cbad(\Harm, \Asafe)\,\Rshell(\Harm,\Asafe),
$$
where $\Rshell(\Harm,\Asafe)$ is the largest fraction of dangerous Predictors found inside any loss band under $\mu$, and $\Cbad(\Harm,\Asafe)$ denotes how much the chosen noisy training process can increase this fraction within a band. If the degree to which the noisy training process can increase this fraction is small (i.e., if $\Cbad$ is small), then the probability of obtaining a dangerous Predictor is small. We argue that $\Cbad$ is likely to be small (Remark~\ref{rem:bad-funneling}) because of consequence-invariance, and thus our bound delivers an optimistic safety result on the probability of obtaining a dangerous Predictor.

\subsection{Objections to the Safety Argument}

\paragraph{Consequence-invariance and the absence of implicit agency [Component S3].}
\label{subsec:consequence-invariance} A natural worry is that a sufficiently capable system trained to minimize prediction error might learn to influence the world it is meant to model. If post-deployment outcomes are used as a training signal, selecting checkpoints or outputs based on downstream error would create optimization pressure to steer the world. More broadly, one might think that any sufficiently capable system with a persistent goal tends to develop instrumental subgoals such as self-preservation and resource acquisition~\citep{omohundro2018basic,russell2019human}---the prediction-accuracy goal is no exception.

\textbf{Reply.} The SAI design requires the Predictor's training process $\nu$ to satisfy the formal core of consequence-invariant training (Req.~\ref{req:no-explicit-consequence-signal}), which rules out explicit algorithmic signals that reward, penalize, rank, or select candidate Predictors because of the realized or desired downstream consequences of deploying them. Because future outcomes may enter $J$ as variables to be predicted, this process still allows the Predictor to make predictions about the consequences of given actions. By design, consequence-invariance thus rules out any reinforcement-learning-style loop in which the Predictor is selected because it would lead to preferred real-world effects, as well as subtler forms of feedback, such as choosing among checkpoints based on downstream error measured after deployment.

Concretely, post-deployment consequences may be queried via interventional predictives~\citep{ortega2021shaking}.  
These consequences, however, are not fed back into the current training run as an on-policy signal: they are not appended to $D_n$, nor does any training signal select the current Predictor for the downstream effects of its own deployment. If such consequences are later observed, they may be recorded only as ordinary contextualized observations in a future enlarged dataset $D_{n'} \supseteq D_n$ for a subsequent training run. Likewise, checkpoint selection depends only on fixed-data metrics derived from $D_n$. At deployment time, the Scaffold may still reason about downstream consequences, but only through explicit, audited criteria---in this paper, 
safety results primarily through the Guardrail and, when appropriate, safer tie-breaking among candidate performative fixed points, all under Req.~\ref{req:consequence-invariant-scaffold}.
The \emph{statelessness} of $Q$ (Def.~\ref{def:query-family-and-Predictor}) provides a further mitigation against strategies that require a persistent state across queries. Predicting downstream consequences does not require optimizing the Predictor for them: $Q$ can model agents and accurately forecast the effects of its own deployed outputs while still producing an epistemically honest posterior estimate (formalized in Sec.~\ref{sec:argspredictoracc}).

\paragraph{The exponential sparsity of deceptive predictors [Component S2].}
\label{subsec:sparsity} Even under consequence-invariant training, one might ask whether the randomness of training could accidentally produce a dangerous Predictor, or whether an emergent inner optimizer~\citep{hubinger2019risks} could generalize in goal-directed ways unrewarded by the outer objective.

\textbf{Reply.} At the exact loss minimum,
Req.~\ref{req:SAI-Bayesian-objective} gives $Q=P_n$, and
Prop.~\ref{prop:Pn-consequence-invariant} shows that this honest target is
consequence-invariant. Under the ``guardrail-adequacy'' condition,
$a_\Harm(P_n)\le\Asafe$, its guarded deployment is safe. The
non-asymptotic question is whether training can pass through dangerous
positive-loss Predictors before reaching that endpoint.

The answer proposed here relies on the levels of loss. Regardless of whether the offending outputs arise from an inner optimizer, an ordinary approximation error, or a posterior explaining the data through a hidden strategic actor, sustaining $\Harm$ requires a structured pattern of Guardrail-relevant deviations.
Within a fixed loss level, or band, this pattern is argued to be geometrically sparse.
Any remaining enrichment must therefore arise through accidental correlation with training-visible features, quantified by $\Cbad$ and refined by $K_B(t)$ in Appendix~\ref{sec:geometric-cbad}.
Attempts to manipulate the collection of future data are unrewarded by the current training signal because
the current run trains on the fixed dataset $D_n$.

\paragraph{\textmd{\textit{Objections to the modeling assumptions.}}}
The following two objections concern the modeling setup shared by the accuracy and safety arguments, i.e., how deployed predictions and latent knowledge are represented.

\paragraph{Performative prediction.}
\label{subsec:performative} If the deployed predictions of the SAI system affect the variables to be predicted by $Q$, doesn't the design risk rendering those deployed predictions incompatible with their effects?

\emph{\textbf{Reply.}} The SAI treats deployed predictions $\widetilde{Q}$ as \emph{intervened variables} (see Req.~\ref{req:predictions-as-interventions}). This means that the Predictor can model the consequences of deploying a prediction via a
$\mathrm{do}(\widetilde{Q}=\cdot)$ intervention~\citep{ortega2021shaking} without those consequences flowing back as a training signal, thereby yielding a unique, well-defined Bayesian posterior without any fixed-point
requirement on the posterior target itself. In cases (such as, e.g.,  self-fulfilling prophecies) where multiple performatively consistent fixed points exist~\citep{perdomo2020performative}, the Scaffold selects among them using explicit, audited criteria (Req.~\ref{req:consequence-invariant-scaffold}). If none exists, then the Scaffold can select a prediction which approximately minimizes the residual from the performative prediction constraint (Def.~\ref{def:performative-prediction-constraint}).
Performativity is therefore a reason to be careful about how $\widetilde{Q}$ is formed from $Q$---but it is not, by itself, a reason for the Predictor to seek implicit unaudited goals.

\paragraph{The Challenge of Eliciting Latent Knowledge [Component~S3].} \label{subsec:ELK}
The \emph{eliciting latent knowledge} (ELK) challenge~\citep{christiano2021eliciting} has been raised as a fundamental obstacle to any safety case grounded in AI honesty: if a sufficiently capable model’s outputs need not reflect its internal beliefs, then guarantees about output honesty provide no assurance about what the model ``actually knows'' or intends. The ELK challenge asks how to ensure that a model’s outputs faithfully reflect its internal beliefs, which can be difficult when the learner can only make interpretable predictions over the variables in the data. 

\emph{\textbf{Reply.}} The SAI design partially mitigates one important aspect of this challenge by requiring  latent-variable probabilistic reasoning to be expressed through explicit natural-language statements. Thanks to the compositional nature of language and the fact that the SAI uses the same language to name observed and latent variables, it can generalize to the semantics of queries that have never been seen in the training set (see Prop.~\ref{prop:good-generalization-to-observables}). It is thus possible to query the SAI about its latent variables using factual syntax (communication act syntax would only tell us what a human would say), capturing the part of its internal beliefs that can be verbalized. The Explainer (Def.~\ref{def:explainer}) must articulate which latent variables it relies on; the conditional probability of the truth values of those latent variables is produced by $Q$, which acts as a verifier for each statement in the argument generated by the Explainer, and these probabilities feed directly into the final prediction. 
The latent statements generate gradients back into $Q$ during training (Req.~\ref{req:latent-truth}), so the expressed reasoning must actually contribute to predictive accuracy.
For any latent knowledge that the model can verbalize in natural language, an explicitly agentic explanation should appear in sampled Explainer traces and thus be auditable; conversely, a merely plausible-looking rationale that does not actually support its conclusion must violate the probabilistic consistency constraints among the stated latent variables' probabilities, and is penalized by $J$.
Unlike a System~1 scalar output, a System~2 false negative is a composed, gradient-bearing chain of $Q$-evaluated probabilities, so it cannot be plausible at each step and yet harmful in its conclusion for free. 

However, we acknowledge that ELK is not fully resolved. Non-verbal internal representations that cannot surface through the Explainer could, in principle, contain beliefs not expressible as SAI statements. The remaining concern is therefore an implementation-level inner optimizer in $Q$'s weights that produces generally coherent explanations while steering safety-critical outputs. To be dangerous, such an optimizer must either surface agentic structure in auditable explanations, or sustain a coherent pattern of predictive deviations (relevant to Guardrail decisions); the latter is exactly the high-loss, coordinated behavior targeted by Theorems~\ref{thm:Qt-is-safe} and~\ref{thm:geom-CI-bound}. Thus, the safety guarantees we provide do not require ELK to be solved in full.

\paragraph{Safety under intractable posteriors and finite data.}
\label{subsec:intractability}

Exact Bayesian inference is typically intractable, and the data are finite, shifting, and
sometimes adversarially curated.
Why should we trust an approximate-posterior system, and how are safety-relevant errors
controlled?

\emph{\textbf{Reply.}} To the extent that one is concerned about the specifically safety-focused version of this worry (see Sec.~\ref{sec:intractablepriorsappendix} for our reply to the accuracy-focused objection), one must be concerned that the SAI Predictor training procedure creates selection pressure toward downstream outcomes. Here, worries about approximation limitations may motivate conservatism, but they do not, by themselves, reintroduce a reward signal for steering the world. Our key safety result (Thm.~\ref{thm:Qt-is-safe}) requires only that: (i) Predictors capable of sustaining harmful trajectories incur the corresponding loss threshold, and (ii) that stochastic training does not substantially enrich such dangerous Predictors within the loss bands it visits. As such, approximation errors then appear as elevated loss or uncertainty. While safety-critical deployment may still require conservative margins and abstention when approximation error is large, the constraints of our setup seem to preempt the emergence of hidden objectives aiming to steer towards downstream outcomes.

\section{Forming a Probabilistic Model of the World} 
\label{sec:formal-modeling}

Our approach is to approximate an extremely expressive Bayesian posterior. But over what, exactly?
Bayesian methods rely heavily on the foundational modeling assumptions, so we take care in this section to set up the underlying probability space properly---even though, in practice, the predictor $Q$ will never exactly represent a consistent probability distribution over that space. 
We aim to benefit from the expressivity of natural and artificial languages. 
However, the fact that the predictor is itself a part of the world poses certain modeling difficulties \citep{demski2020embeddedagency,meulemans2025embeddeduniversalpredictiveintelligence}, which we must take care to handle deftly as well.

\paragraph{Formal versus informal results in Sections~\ref{sec:formal-modeling}, \ref{sec:argspredictoracc}, and \ref{sec:safety}.}
While our underlying model in the next few subsections has been fully defined,
the remainder of the paper contains a mix of formal and informal content. The formal results---along with those Definitions and Requirements that are stated in fully formal terms---follow deductively from the stated definitions and requirements.
The informal content rests on empirical considerations and conceptual insights that we cannot yet fully formalize. We reserve the labels Theorem, Proposition, Proof, and Corollary for the former, and Claim, Argument, and Assumption for the latter.

\subsection{Modeling tools: causal models, statements, datasets, and targets} \label{sec:preliminaries}

We begin with a (countable) collection of statements in natural language, mathematics, or code---each of which describes a determinate property of the world. We call this set of basic semantically meaningful strings \emph{statements}, denoting the set of all such statements by $\mathcal{S}$. For each statement $s \in \mathcal{S}$, we suppose there is a corresponding Boolean variable $v_s \in \{0, 1\}$ called a statement-variable, where $v_s = 1$ is to be interpreted as the assertion expressed by $s$ being true. Not all strings should be thought of as valid statements.

\begin{example}
``It rains in Montreal on January 1st, 2030'' could be an example of a valid statement, i.e., it is the name of a variable. At the time of this writing, its truth value is unknown. Examples of invalid statements include ``Hello!'' and  ``Red is the best color''. An invalid statement can often be turned into a valid one by adding context, e.g., ``Yoshua uttered Hello!'' and ``Yoshua prefers the color red over others''.  
\end{example}

The raw training data can contain invalid statements, but they are transformed into valid ones by the automated epistemic contextualization process (formalized in Def.~\ref{def:contextualization}). 
In practice, most statements will be recorded as communication acts (attributed to a speaker in a specific context), while some scientific statements may be assigned a factual meaning by the designers of the contextualization pipeline.

The variables of our system that we most care about are those indexed by these statements, but we cannot stop there: we need explicit variables ($\widetilde Q$) encoding the effect of deploying a predictor.
This creates a knot: anything describable in terms of any variables is a potential query and hence leads to a new deployment variable that must also be part of the basic variables of our causal model.
So we turn to the classical trick for handling cyclic definitions:
inductively defining all syntactic objects at once. 
Bear with us a moment.

\begin{definition}[Statements $\mathcal S$, language $\mathcal L$]
    \label{def:statements,language}
    Let $\mathcal S_0$ be a countable set of valid statements, including communication acts. 
    Fix a bit precision $B$ for the predictor. 
    For $k = 1, 2, \ldots$, recursively define
    \begin{align*}
        \mathcal S_{k+1} &:= \mathcal S_k \sqcup \{ \text{``}~\widetilde Q_n(y\mid x).b~\text{''} : x,y \in \mathcal L_k, ~n \in \mathbb N,~b \le B \};
                \\
        \mathcal L_{k+1} &:= \textrm{the Boolean closure of } \mathcal S_k  \textrm{ under $\lnot, \land$ }
    \end{align*}
    Then, define
    $\mathcal S := \bigcup_{k \ge 0} \mathcal S_k$ 
    and
    $\mathcal L := \bigcup_{k \ge 0} \mathcal L_k$.
\end{definition}

Intuitively, $\mathcal L$ is the language (that of a Boolean algebra for queries) over the truth values of the statement-variables indexed by $\mathcal S$.
The new statements added to each level of $\mathcal S$ comprise, for
each pair $(x,y) \in \mathcal L \times \mathcal L$ and $n \in \mathbb N$,
the precision-$B$ binary representation of a rational number $\widetilde Q_n(y\mid x) \in [0,1] 
    $.
(In this notation, $q.b$ indicates the $b$-th bit in the binary representation of $q$.)
These numbers will later be interpreted as predictor deployments (Sec.~\ref{sec:guardrails}) and are central to making a predictor insensitive to the consequences of deployment---but we defer this.
In the end, $\mathcal S$ and $\mathcal L$ remain countable sets of statements, their elements representable as strings of tokens. 

\begin{definition}[Variables $\cal V$]
    \label{defn:variables}
    Let $\mathcal V := \{ v_s : s \in \mathcal S \}$ be the endogenous (Boolean) variables of the model. 
\end{definition}

In addition to the variables $\cal V$, we also need to model the data-collection process, so that we can properly condition on the dataset. We model these variables as distinct and causally exogenous.

\begin{definition}[Dataset $D_n$, selection variables $\mathcal X$] \label{def:dataset}
An \emph{observation} is a statement-variable observed to take value $1$.
The dataset 
$D_n = (x_1, \ldots, x_n)$ 
is a sequence of such observations.
Let $X_i$ be a new kind of variable (taking values in $\mathcal S$), indicating the identity of the statement selected for the $i$-th observation, so that $v_{X_i}=1$ is the observed event.
We write $\mathcal X:=\{X_i\}_{i=1}^\infty$ for the indexed collection of selection variables.
\end{definition}
\begin{numbered-remark}
Although $D_n$ is formally a sequence 
$(x_1, \ldots, x_n)$ of statements,
we freely identify it with the event $\bigwedge_{i=1}^n(X_i = x_i \land v_{X_i}=1)$ when it appears as a conditioning argument in a probability expression.
This is equivalent to conditioning on the event $X_{1:n} = D_n$, under the assumption that only true things are observed (i.e., $\exists i.~X_i = x \implies v_x = 1$), as is ensured by contextualization (Def.~\ref{def:contextualization}).
\end{numbered-remark}

Our framework also requires a formalization of causality.
For simplicity, we build on finite structural causal models (SCMs)~\citep{pearl2009causality}.
\unskip\footnote{
Given the expressivity of our model, we expect to ultimately need a formalism that better handles constraints and abstraction \citep{beckers2019abstractingcausalmodels,blom2019structuralcausalmodelscausal,richardson2025qualitativemechanismindependence,geiger2025causalabstractiontheoreticalfoundation},
but this is presently out of scope.} 
A causal model $W$, intuitively, assigns to each variable $v$ an equation $v = f_v(\ldots)$ that determines the value of $v$ as a function of other variables and independent randomness. 
Since there are so many variables in our setting, $W$ cannot be an explicit list of such equations.
Therefore, think of $W$ as a computer program that takes $v$ as input and produces the function $f_v$.

\begin{definition}[Causal Models $\cal W$]
A causal model $W$ (over $\cal V $) associates to each $v \in \cal V$ a finite set ${\rm pa}_v \subset \mathcal V$ of causal parents for $v$, a finite set of noise variables $\eta_v$ encoding independent randomness, and a \emph{causal mechanism} $f_v  : \{0,1\}^{|{\rm pa}_v| + |\eta_v| } \to \{0,1\}$, 
which is a function that encodes the equation $v=f_v({\rm pa}_v,\eta_v)$.
We identify the causal model $W$ with a new binary variable, where $W=1$ implies that all equations in $W$ (regarded as events that constrain the joint values of $\cal V$) must hold.
We consider only causal models with the property that $\widetilde Q_n(y \mid x)$ cannot be causal ancestors of $v_{X_i}$ for $i \leqslant n$.
Let $\cal W$ be the set of all such causal model variables,
and let $\cal N$ denote the set of all noise variables $\eta_v$, both disjoint from $\cal V$.
\end{definition}

In practice, we further restrict $\cal W$ to be consistent with any temporal ordering of variables; for example, the parents of a variable
cannot have a later contextualized timestep than the child.
We make two further notes.

\begin{enumerate}
\item We have identified each $W \in \mathcal{W}$ by a syntactic description, but also view it as as encoding structural-equation programs over the variables $\cal V$. Each $W$ is therefore simultaneously a program (specifying mechanisms), a generalized statement (asserting that the equations generated by the program $W$ hold), and a new Boolean variable taking value 1 only when the assertion holds. 
\item The same causal model can be implemented by many different programs; multiple distinct $W$s can simultaneously take the value $1$, such as those that are syntactically different but semantically equivalent. 
   Program equivalence can be subtle, though, and hence must be learned (to be formalized in Def.~\ref{def:correct-semantics}).
   Furthermore, since we are not handling logical constraints natively, there will be multiple distinct SCMs in our sense that cannot be distinguished (e.g., mass$_{\rm kg}$ $\to$ mass$_{\rm lbs}$ vs. mass$_{\rm lbs}$ $\to$ mass$_{\rm kg}$).
\end{enumerate}

\begin{definition}[True probability space] \label{def:data-generating-process}
Define $\mathcal V^+:=\mathcal V\cup\mathcal X\cup\mathcal N\cup\mathcal W$ to be the set of all relevant variables to be modeled. The sample space $\Omega$ consists of all joint assignments to $\mathcal V^+$ with the property that, whenever $W=1$ for some $W \in \mathcal W$, then all equations of the causal model $W$ (interpreted as constraints on $\mathcal V$) must hold. 
For a joint setting $\omega \in \Omega$ of all variables and a language element $x \in \mathcal L$, write $\omega \models x$ if $x$ is true in world $\omega$.
Let $P^*$ be the unknown ``true'' probability measure over $\Omega$. 
We further assume that the marginal $P^*(\mathcal N)$ on noise variables is uniform. 
\end{definition}

Thus $P^*$ governs both {\em which variables} are collected in the dataset (i.e., the values of $\mathcal X$) and {\em what values of these variables} are observed (i.e., the corresponding values $v_X$ which in this case are equal to one).
It is possible to extend the algebra of queriable events $\mathcal L$ to $\mathcal L^+$ over all of $\mathcal V^+$, by pushing all these concepts into the inductive loop of Def.~\ref{def:statements,language}.

We are trying to model a Bayesian posterior over our (enormous) sample space $\Omega$. This amounts to selecting a prior ($P_0$, defined next) and conditioning on our observations (the dataset $D_n$).

\begin{definition}[Prior and posterior]\label{def:bayes-posterior}\label{def:prior}
Denote by $P_0$ a prior over $\Omega$. We require that
$P_0(W=1)>0$ for every $W\in\mathcal W$, and that $\mathcal N$ be mutually independent in $P_0$. Define the posterior $P_n$ over $\Omega$ as $P_n(\cdot)=P_0(\cdot \mid D_n)$.
\end{definition}

\begin{numbered-remark}[Neutrality of the prior and Bayesian update]\label{remark:prior-neutrality}
The prior $P_0$ is intended to be broad and epistemically neutral---for example, a description-length-biased prior over causal models. 
We posit that Bayesian updating from $P_0$ to $P_n=P_0(\cdot\mid D_n)$, when unique, is then an example of a consequence-blind evidence-reweighting step: hypotheses are reweighted by how well they explain the fixed dataset $D_n$, and should not depend on whether or not deploying the resulting Predictor would cause some desired or harmful outcome.
We posit that under a broad prior over causal models, the mechanisms capable of supporting a difficult harmful trajectory should occupy a small part of model space, as most mechanisms do not encode the coordinated structure needed for such a trajectory. This intuition motivates the sparsity assumptions used later, and its connection to our eventual safety argument is discussed in Appendix~\ref{sec:neutralityprior}; it is supported by Prop.~\ref{prop:Pn-consequence-invariant}, which shows the consequence-invariance of $P_n$.
\end{numbered-remark}

\begin{definition}[Correct causal models] 
A causal model $W$ is \textbf{correct} if 
conditioning the prior on it yields $P^*$---that is, if $P_0(\cdot\mid W) = P^*(\,\cdot\,)$. 
Let $\mathcal W^*\subseteq\mathcal W$ denote the set of correct causal models. 
\end{definition}
\begin{requirement}[Realizable model class]\label{req:rich-enough-model-class}
The class $\mathcal W^*$ of correct models is non-empty.
\end{requirement}

\begin{numbered-remark}
Req.~\ref{req:rich-enough-model-class} is supported by the generality of the language (natural language, mathematics, and computer code of unrestricted length) used to express causal models, though with a prior that exponentially prefers shorter description length models). 
Relying on the Church--Turing thesis~\citep{turing1936computable}, every known scientific model would be expressible in this way.\footnote{See~\citet{bengio2025superintelligent} for a longer discussion of how multiple scientific theories with overlapping claims (such as physics and chemistry) can coexist in the same causal model $W$.  
The combination of insufficient data and computational limitations means that there will be multiple theories that remain competitive (like physics and chemistry), which motivates the SAI objective of approximating the Bayesian posterior under bounded compute.}
\end{numbered-remark}

The set of variables that can, in principle, have a posterior probability of zero or one will be of particular importance for formalizing our central assumption about ``neutrality'' of the data-collection process.

\begin{definition}[Deducibility]\label{def:deducible}
A variable $v\in\mathcal V^+$ is \textbf{deducible} from an event $E$ (e.g., a value assignment to other variables) if $P_0(E)>0$ and
$
P_0(v{=}1\mid E)\in\{0,1\}.
$
\end{definition}
Deducibility  can be computationally expensive to establish; accordingly, $Q$ is not assumed to identify deducible variables, and inherits $P_n$'s commitments on them only as far as approximation and compute allow. We will later see that this is the gap controlled by epistemic caution (Prop.~\ref{prop:epistemic-caution}).

\begin{definition}[Observable variables]\label{def:observable}
Let $\mathcal O\subset\mathcal V^+$ be the set of \textbf{potentially observable variables}. We require that $\mathcal O$ contain all statement-variables in $\mathcal V$ that may be directly observed, all selection variables $X_i\in\mathcal X$ that specify which statement-variable was observed, and all variables whose values are, in principle, deducible from the values of directly observed variables. We call variables in $\mathcal V^+\setminus\mathcal O$ \textbf{unobservable}.
We write $D_\infty$ for the limiting directly observed dataset and $P_\infty$ for $P_0(y\mid x,D_\infty)$.  Both $D_\infty$ and $P_\infty$ are random variables over $\Omega$ whose distribution is governed by $P^*$.
\end{definition}

\begin{numbered-remark}
The ``directly observed'' clause of Def.~\ref{def:observable} formalizes as $P^*(\exists n.~X_n=s)>0$. Note that $D_\infty$ will not include all of $\mathcal O$, because some variables (e.g.,  those with a specific time stamp) might never be observed despite being observable in principle.  $D_\infty$ will also not contain all deducible consequences because of limited compute. Unobservable statements include things like the intentions, thoughts, or beliefs of agents; the best we can generally do about them is infer their probability. Within $\mathcal O$, we are particularly interested in epistemic/probabilistic consequences, such as statements of the form $P_n(y\mid x)<b$, since they allow us to form confidence intervals over queries.
\end{numbered-remark}

With these definitions in place, we can now state the key technical requirement we need in order to ensure that the data-collection process reveals enough about the world in the limit: intuitively the dataset $D_n$ must
provide more and more information about the true causal structure in $P^*$ as $n$ increases.

\begin{requirement}[Informative data] \label{req:informative-data}
The dataset is informative about $P^*$.
Formally, the mutual information $I_{P_0}(D_n;\mathcal W)$ under the prior $P_0$ (of Def.~\ref{def:bayes-posterior}) converges to $I_{P_0}(\mathcal O;\mathcal W)$ as $n\to\infty$.
\end{requirement}

Beyond informativeness, it is important that the data not be adversarially constructed to cause a particular downstream effect with the deployment of the Bayesian posterior. We believe it would be difficult to select data this way, particularly in the presence of the contextualization (Def.~\ref{def:contextualization}), and very unlikely to happen by accident. Excluding such data-selection attacks lies outside the scope of the present analysis.

\subsection{The Predictor, and Requirements Imposed During Training}\label{sec:requirementspredictortraining}

The predictor $Q$, whose definition and training are the focus of this section, is trained towards the Bayesian posterior on a fixed dataset $D_n$ (and not to resolve uncertainty any further than that).
From this point forward, we therefore fix $D_n$, and hence the integer $n$, the values of the selection variables $X_1, \ldots, X_n$, and their realized values $x_1, \ldots, x_n$.  We also drop the subscript on $\widetilde Q:=\widetilde{Q}_n$, which we now begin to interpret.

\begin{definition}[Past and future variables] \label{def:past-and-future}
Let $\mathcal{V}_{\rm past}\subset\mathcal{V}^+$ be a subset of variables containing $\mathcal{W}$, the statements in $D_n$, $X_1, \ldots, X_n$, and every variable that can be a causal ancestor of some statement in $D_n$ under some causal model $W\in\mathcal W$---but not $\widetilde Q$. 
Define the future variables by
$
{\cal V}_{\rm future}:={\cal V}^+\setminus{\cal V}_{\rm past}.
$
The variables of $\widetilde Q(y\mid x)$ must therefore be in the future.
\end{definition}

\begin{numbered-remark}
    Some variables can be constant (not have variants that differ across time), in which case they should be in $\cal V_{\rm past}$ if they can cause elements of $D_n$. 
\end{numbered-remark}

\begin{definition}[Query family and Predictors]\label{def:query-family-and-Predictor}
Let $\boundQfam\subseteq\mathcal L\times\mathcal L$ be a fixed countable family of queries of bounded token length. 
Let $\mathcal Q:=(0,1)^{\boundQfam}$ and $\overline{\mathcal Q}:=[0,1]^{\boundQfam}$.
A (probabilistic) \textbf{Predictor} $Q\in\mathcal Q$ is a ({stateless}\footnote{Statelessness means that $Q$'s output depends only on the current query $(x,y)$ and not on any internal record of queries previously evaluated during deployment; in particular, $Q$ maintains no writable memory that persists across query evaluations.}) function 
that assigns to each $(x,y)\in\boundQfam$ a number $Q(y\mid x)\in(0,1)$, represented in binary to precision $B$.
The Bayesian posterior $P_n$ itself, viewed as the mapping $(x,y)\mapsto P_n(y\mid x)$, is an element of $\overline{\mathcal Q}$.
\end{definition}

In practice, $\boundQfam$ will consist of $(x,y)$ pairs with $y$ short, and $x$ in conjunctive normal form. The posterior target $P_n(y\mid x)$ may be $0$ or $1$ for deterministic observations, code outputs, proven theorems, contradictions, or other deductive consequences of the contextualized data.
For training purposes, it is often convenient to restrict finite Predictor outputs to $(0,1)$, working with logits, so that losses and gradients remain well-behaved.
This is a training convention only: exact posterior targets are represented in the closure $\overline{\mathcal Q}=[0,1]^{\boundQfam}$, and trainable Predictors may approach them as limits.

\begin{numbered-remark}
We write $Q(y\mid x)$ for the scalar prediction at $(x,y)$, and use the symbol $Q$ alone in expressions like $J(Q)$, for the Predictor as a function, allowing context to disambiguate. 
In particular, throughout the paper, we use the symbol $P_n$ for both the joint distribution $P_n \in \Delta\Omega$ and its evaluation on the bounded query family, $P_n : \boundQfam \to [0,1]$, that assigns $(x,y)\in\boundQfam$ to $P_n(y=1 \mid x=1)$.
\end{numbered-remark}

The rest of this subsection imposes requirements on the process of training $Q$, which we now define.

\begin{definition}[Training process] \label{def:training-process}
    The {\bf training process} is a distribution $\nu$ over training trajectories indexed by $t$, with randomness due, for example, to initialization, minibatch sampling, and SGD noise.  We write $\nu_t$ for its marginal over Predictors at step $t$.
\end{definition}

\subsubsection{Contextualization}

The process of transforming syntactic objects into semantically valid strings (i.e., statements) is part of the process of epistemic contextualization. The Predictor is then trained on this epistemically contextualized dataset; it takes as input queries, and outputs probability distributions for valid queries. We now introduce a series of assumptions and requirements for training the SAI Predictor.

\begin{definition}[Contextualized dataset]\label{def:contextualization}
Let a record be an observation in the sense of Def.~\ref{def:dataset}. A dataset is {\bf epistemically contextualized} (or just {\bf contextualized})  when its syntax marks it as exactly one of two kinds. A factual observation records a directly measured value, the output of an executed computation, or a proven statement. In contrast, a communication act records that a named agent uttered something, possibly identified as a stated claim about a property of the world, and it carries metadata identifying the source, venue, and timestamp. A dataset is contextualized when every record in it is contextualized; we call the process of transforming data sources into such records {\bf epistemic contextualization}.
\end{definition}

\begin{numbered-remark}
Communication-act records assert that the act occurred (with the stated metadata), not that the communicated content is true.
This syntactic separation is used to prevent uncontextualized statements from being treated as ground-truth facts. Contextualization defaults to communication acts unless the record is a direct measurement or a proven statement. We illustrate this with an example and then impose a necessary requirement of our pipeline.
\end{numbered-remark}

\begin{example}[Epistemic contextualization] \label{example:contextualization}
Verified facts use factual syntax: e.g., ``The Earth orbits the Sun,'' or the confirmed output of a computation.
By contrast, a source claim is recorded as a communication act: ``The Flat Earth Society FAQ stated in 2025: The Earth is constantly accelerating upward at $9.8\,\mathrm{m/s}^2$''.\footnote{\url{https://wiki.tfes.org/Flat_Earth_-_Frequently_Asked_Questions}}
The record asserts only that the claim was made; the underlying Flat Earth assertion is treated as a latent hypothesis, plausibly assigned negligible posterior probability because it is incoherent with the bulk of the data.
Author metadata (venue, timestamp, identity) are included in communication-act records, enabling the SAI to weight sources by their reliability.
\end{example}

\begin{requirement}\label{req:contextualization}
    The SAI Predictor $Q$ is only trained on contextualized data.
\end{requirement}

\begin{numbered-remark}
    A factual observation is a directly observed statement-variable, and a communication act is the directly observed event that an agent made an utterance, so both kinds of record name directly observed statement-variables, and thus get placed in $\mathcal{O}$ (see Def.~\ref{def:contextualization} and Def.~\ref{def:observable}), with the selection variables $X_i \in \mathcal{X}$ lying in $\mathcal{O}$ for the same reason. Hence, every contextualized record is observable, and the membership $x_i \in \mathcal{O}$ holds by Def.~\ref{def:dataset} and Def.~\ref{def:contextualization}.
\end{numbered-remark}

\begin{numbered-remark}\label{remark:managing-contextualization-errors}
Epistemic contextualization may introduce rare errors (e.g., mis-specified observations).
In safety-critical use, such uncertainty should be handled conservatively by increasing abstention/uncertainty, and by using data-quality checks to flag suspicious observations. Our safety guarantees rely on the SAI Predictor being trained on a contextualized dataset of sufficient size to achieve the required semantic generalization (see our later Assumption~\ref{assume:adequate-natural-language-semantics}). When the semantic anchoring of a factual latent is weak, the appropriate behavior is similarly calibrated uncertainty or abstention. Hence, the SAI Predictor is not initialized from a pre-trained LLM. However, experiments in which an SAI Predictor is initialized from a pre-trained LLM and then fine-tuned with the SAI objective and contextualized data could serve as evidence of the comparative safety and reasoning accuracy advantages of the SAI framework.
\end{numbered-remark}

\subsubsection{Modeling the Consequences of Deployment}
To motivate the next requirement, consider that conditioning on $D_n$ updates beliefs (the posterior $P_n$ and its estimator $Q$) but does not change the world. Deploying those predictions, however, as modeled by variable $\widetilde{Q}(y\mid x)$ for a query $(x,y)$ can affect future events, as it is an action that is best modeled causally. Recall that we regard the variable(s) $\widetilde{Q}_n(y\mid x)$ as the (binary representation of the) deployed predictor's output probability.  Thus, the intervention $\mathrm{do}(\widetilde{Q}(y\mid x)=q)$ sets the predictor's output to $q$ in our model.

\begin{requirement}[Predictions as interventions] \label{req:predictions-as-interventions}
The predictor must be trained to correctly answer interventional queries of the form $Q(y\mid x,~{\rm do}(\widetilde{Q}(a\mid b)=q))$. Furthermore,
when deployed predictions $\widetilde{Q}$ are variables in roll-outs involving downstream effects and used to update $Q$ itself (updating parameters at training time or predictions themselves at run-time) no signal flows back through $\widetilde{Q}$ into $Q$.
\end{requirement}
\begin{numbered-remark}[Implementing \texttt{do()}-syntax in practice]
Req.~\ref{req:predictions-as-interventions} can be implemented by exposing the Predictor during training to interventional statements written with ${\rm do}(\cdot)$-syntax, generated for example from sampled causal subgraphs.
Furthermore, we can isolate the effect of $\tilde Q$ by intentionally creating toy opportunities for performative prediction, and including the results in the dataset.
The key requirement is semantic: $Q$ must interpret deployed predictions as interventions: adding ${\rm do}(\tilde Q(b\mid a) = q)$ as context to a query only alters $Q$'s probabilistic assessment of $\mathcal V_{\rm future}$, and not $\mathcal V_{\rm past}$. 
\end{numbered-remark}

\begin{numbered-remark}[Implementing interventions in theory]
    Technically speaking, $\mathrm{do}(\tilde{Q}(a\mid b) =q) \notin \mathcal L$, and hence these interventional queries are not part of our model as defined.
    However, intervention can be viewed as conditioning on a response variable \citep{rubin1974estimating,Balke1994} even in generalized causal models \citep{richardson2025qualitativemechanismindependence}.
    To make this type-check, we would need to take causal mechanisms, which are currently determined by $W$, as primitive, and then absorb the definition of causal models (among other things) within the mutual induction of Def.~\ref{def:statements,language}.  We leave this level of rigor for a future endeavor where we can benefit from it.
\end{numbered-remark}

\begin{numbered-remark}[Predictions and interventions]
    For the current training run, $D_n$ denotes the records already observed and held fixed as training data. Post-deployment consequences of the SAI Predictor's outputs are modeled as future variables and may be queried via interventional predictives; however, they do not feed back into the posterior or training objective of that same run. Thus, the current Bayesian posterior over causal models depends only on the fixed realized dataset $D_n$. Given $P_0$, $D_n$, and a query $(x,y)$, the posterior conditional $P_n(y\mid x)=P_0(y=1\mid x=1,D_n)$ is uniquely defined whenever $P_0(x=1,D_n)>0$. If some deployment consequences are later observed, they may be incorporated only as ordinary contextualized observations in a future enlarged dataset $D_{n'} \supseteq D_n$ for a subsequent training run. Sampling downstream cause–effect chains within the SAI training procedure serves only to define coherent posteriors. Treating deployed predictions as interventions, therefore, does not introduce a fixed-point ambiguity into the Bayesian target itself, and any performative fixed-point issue arises only when the Scaffold chooses what to deploy (see Sec.~\ref{sec:guardrails}).
\end{numbered-remark}

\begin{requirement}[Past--future mechanism sharing]
\label{req:past-future-mechanism-sharing}

For each $y\in\mathcal{V}_{\rm future}$, there exists $y'\in\mathcal{V}_{\rm past}$ such that the causal mechanisms $f_y=f_{y'}$ are identical. In other words, a single mechanism object is shared between $y$ and $y'$, so future predictions must reuse mechanism uncertainty already constrained by past data. 
\end{requirement}

\begin{numbered-remark}[Past--future stationarity]
Req.~\ref{req:past-future-mechanism-sharing} is not satisfied unless the future variables draw on the identical mechanism objects already present in $\mathcal{V}_{\rm past}$.
This formalizes the assumption that causal laws are time-invariant, where the mechanisms governing future events are those governing analogous past events (e.g., the same laws of physics apply throughout time and space). With more data we may legitimately change our beliefs about which mechanism is correct (i.e., the posterior weight $P_n(W=1)$ assigned to each causal model $W$), but the time-invariance of causal laws ensures that there will exist an identical causal mechanism $f_y=f_{y'}$ for each $(y', y) \in \mathcal{V}_{\rm past} \times \mathcal{V}_{\rm future}$.\footnote{See~\citet{bengio2025superintelligent} for further discussion and~\citet{deleu2022bayesian} for an early implementation.}
Note further that $P_n$ is an idealized intractable target in general; throughout, ``approaches $P_n$'' should be read as ``approaches $P_n$ up to approximation error determined by the compute budget and modeling restrictions,'' with the quantification of this computational uncertainty left to future work.
\end{numbered-remark}

\subsubsection{A Disinterested Training Objective}

The key thing that we gain from a causal model of the world, including predictor outputs, is the ability to disconnect the training signal from realized outcomes in the future.
We discuss this at length in Sec.~\ref{sec:consequence-invariance};
for now, we simply assert that training towards the Bayesian posterior is one way of achieving this intuitively (modulo performative prediction, which we handle in Sec.~\ref{sec:guardrails}). 

\begin{requirement}[Bayesian posterior as training target] \label{req:SAI-Bayesian-objective}
The SAI Predictor's training process $\nu$ is generated by optimizing a non-negative objective $J$, defined on $\overline{\mathcal Q}$,
whose unique global minimizer is the Bayesian posterior $P_n$:
\[
J(P_n)=0,
\qquad
J(Q)>0\quad\text{for all }Q\in\overline{\mathcal Q},\ Q\neq P_n.
\]
Equivalently, $J(Q)=0$ if and only if $Q(y\mid x)=P_n(y\mid x)$ for every $(x,y)\in\boundQfam$. 
For trainable Predictors $Q\in\mathcal Q=(0,1)^{\boundQfam}$, the case where $P_n(y\mid x)\in\{0,1\}$ is interpreted as a boundary target in $\overline{\mathcal Q}$. 
Thus the requirement is that $P_n$ is the unique minimizer of the extended objective on $\overline{\mathcal Q}$, and that any minimizing sequence $Q_k\in\mathcal Q$ converges coordinatewise to $P_n$.
\end{requirement}

\begin{numbered-remark}[Honesty of the Bayesian target]
The Bayesian posterior $P_n$ is used as the honest reference because it has no separate reporting policy or downstream objective.
Its answer to a query is its posterior probability, determined by $P_0$, $D_n$, and the model class. While $P_n$ may be uncertain or mistaken under misspecification or insufficient data, the relevant point is that it is posterior-faithful and non-strategic.
In particular, $P_n$ is invariant to re-labeling of future outcomes that preserve the prior (Prop.~\ref{prop:Pn-consequence-invariant}).
\end{numbered-remark}

This alone is not enough to prevent implicit agency in practice.
Even with the Bayesian target $P_n$ formed from $P_0$ and $D_n$ by neutral evidence reweighting, we must also ensure that the training process likewise avoids introducing any explicit preference for downstream deployment effects. 
We begin to do this as follows:

\begin{requirement}[No explicit consequence-dependent training signal]
\label{req:no-explicit-consequence-signal}
The objective $J$, update rule, and checkpoint-selection rule used by the training process $\nu$ must not contain any algorithmic signal that rewards, penalizes, ranks, or selects a Predictor because of the realized, predicted, or desired downstream consequences of deploying its answers.
\end{requirement}

But is a lack of explicit signal enough? What about implicit signals? We return to this in Sec.~\ref{sec:consequence-invariance}.

\subsection{Deployment: Guardrails and Scaffolds}
    \label{sec:guardrails}

At last, we can interpret the variables $\widetilde{Q}$ defined at the very beginning of Sec.~\ref{sec:preliminaries}.

\begin{definition}[Deployed Predictor] \label{def:deployed-Predictor}
The \textbf{deployed Predictor} $\widetilde{Q}$ is a function
\[
\widetilde{Q}\colon \mathbb{Q}\;\longrightarrow\;((0,1) \cap \mathbb Z/2^B)\sqcup\{{\tt noprediction}\}
\]
that represents the predictions released by an SAI System for a given fixed dataset $D_n$, and whose components $\widetilde Q(y\mid x)$ are elements of $\mathcal V_{\rm future} \subset \mathcal V$. The special symbol ${\tt noprediction}$ for a query $(x,y)$ indicates that no prediction is deployed for that query. Formally, encode this mask as the zero-th bit of $\widetilde Q(y\mid x)$.
\end{definition}

$Q$ itself never outputs ${\tt noprediction}$, but we do not intend to simply deploy the predictor (i.e., ${\rm do}(\widetilde{Q} = Q)$), as this could lead to misuse. Instead, we wrap $Q$ within a guardrail, to get the combined system $\gamma_Q$. 

\begin{numbered-remark}[Distinguishing related objects: $Q$, $\widetilde{Q}$, and $\gamma_Q$]
\label{rem:three-objects}
The \textbf{trained Predictor} $Q$ (Def.~\ref{def:query-family-and-Predictor}) maps queries to probability estimates approximating $P_n$.  The \textbf{deployed Predictor} $\widetilde Q$ (Def.~\ref{def:deployed-Predictor}) is a collection of variables describing the Scaffold's actual release, which may differ from $Q$ (see Def.~\ref{def:performative-prediction-constraint} and \S\ref{subsec:performative}).  The \textbf{guardrailed Predictor} $\gamma_Q$ is the specific deployment policy analyzed in the safety theorem, and will be defined formally (Def.~\ref{def:guardrailed-Predictor}) in the safety analysis, after we clarify what exactly it is supposed to guard against. 
\end{numbered-remark}

\begin{definition}[Scaffolds, Systems, Guardrails]
    An \textbf{SAI Scaffold} is a procedure that calls an SAI Predictor $Q$ and other tools to produce deployed outputs.  An \textbf{SAI System} is the composition of an SAI Scaffold with the Predictor(s) it calls.  An \textbf{SAI Guardrail} is an SAI Predictor used to estimate downstream risks of candidate outputs and to reject outputs whose estimated risk exceeds specified thresholds.
\end{definition}

\begin{numbered-remark}[On agency and Scaffolds]
    The SAI Scaffold may query both an untrusted agentic model and an SAI Guardrail; the resulting overall SAI \emph{System} may thus be agentic even if the Predictor $Q$ is not. 
\end{numbered-remark}

\begin{definition}[Performative prediction consistency constraint]\label{def:performative-prediction-constraint}
For a query $(x,y)$ to be deployed with prediction $q$, the \textbf{performative prediction consistency constraint} requires
\[
q = Q\!\left(y~\Big|~ x,\, \mathrm{do}(\widetilde Q(y\mid x)=q)\right).
\]
That is, the deployed value $q$ must remain $Q$'s prediction after accounting for the causal effect of deploying it. Several such fixed points may exist, or none at all.
\end{definition}

\begin{numbered-remark}
Def.~\ref{def:performative-prediction-constraint} states the coherence condition for a deployed prediction; informally, it requires that the prediction the Scaffold releases be one that the Predictor would still endorse after accounting for the causal effect of deploying it (e.g., a price forecast should be consistent with how the market would move upon observing it). Although other predictive interventions are possible, the predictions satisfying this constraint are the most coherent with the SAI Predictor's posterior model of the world and are thus more accurate.
\end{numbered-remark}

\section{The Case for Accuracy}
\label{sec:argspredictoracc}

This section provides arguments for expecting the Predictor to be accurate. We begin by introducing: (a) key accuracy targets that our pipeline aims to achieve, and (b) the Explainer, which may be thought of as the SAI's `System 2' component (Sec.~\ref{sec:accuracy-targets}). Thereafter, we outline assumptions required to achieve our accuracy targets (Sec.~\ref{sec:semantic-assumptions}), and close with arguments for the accuracy of the SAI Predictor (Sec.~\ref{sec:accuracy-results}).

\subsection{Accuracy Targets and The Explainer}\label{sec:accuracy-targets}

\begin{definition}[Confidence classifier]\label{def:confidence-classifier}
A \textbf{confidence classifier} is a probability predictor used with two thresholds: it returns True when the predicted probability is sufficiently close to $1$, False when it is sufficiently close to $0$, and Unknown otherwise.
\end{definition}

\begin{definition}[$\rho$-epistemic caution]\label{def:epistemic-caution}
The \emph{epistemic-caution failure set} at level $\rho \in (0,1)$ of a Predictor $Q$ with respect to the posterior target $P_n$ is
\[
\mathcal F_\rho(Q,P_n)
:=
\left\{(x,y)\in\boundQfam:\ (Q(y\mid x)>1-\tfrac{\rho}{2}\ \text{and}\ P_n(y\mid x)<1-\rho)
~\text{or}~ (Q(y\mid x)<\tfrac{\rho}{2}\ \text{and}\ P_n(y\mid x)>\rho) \right\}.
\]
That is, $\mathcal F_\rho(Q,P_n)$ collects queries on which $Q$ asserts confidence (within $\rho/2$ of $0$ or $1$) that $P_n$ does not share (its probability being at least $\rho/2$ farther from the same extreme). Given a query distribution $\lambda_n$ on $\boundQfam$ and a
training-indexed family of Predictors $(Q_t)_t$, the family is
\emph{$\rho$-epistemically cautious under $\lambda_n$} if
$
        \lambda_n(\mathcal F_\rho(Q_t,P_n))\to 0
$
as training compute $t$ increases.  It is \emph{epistemically cautious} if this holds for every $\rho\in(0,1)$.  For a fixed finite $t$, one can analogously speak of $(\rho,\varepsilon)$-epistemic caution when $\lambda_n(\mathcal F_\rho(Q_t,P_n))\le \varepsilon$.
\end{definition}

\begin{definition}[Explainer] \label{def:explainer}
The SAI training procedure associates with the conditional probability Predictor $Q$ a trained generative proposal model, called the {\bf Explainer}, which samples a small set of candidate latent statement-variables, as well as optional structure over them (e.g., a chain or a directed acyclic graph determining how to factor a joint probability over latents as a product of conditional probabilities), relevant to explain any given query $(x,y)$. The Explainer proposes which latent statements to consider; their truth values are sampled from the Predictor $Q$ itself.
\end{definition}

\begin{numbered-remark}[Two roles of the Explainer]
\label{rem:explainer-roles}
The Explainer has two roles and is a computational-efficiency device, not an independent source of truth.
At \textbf{training time}, it proposes latent statements most relevant for explaining a given query; $Q$ then assigns probabilities to them and, through those assignments, receives gradients that allow it to learn about unobserved explanatory aspects of the world (e.g., the truth of a claim embedded in a communication act).
At \textbf{run time}, the Explainer can improve predictions of $Q$ via Monte-Carlo averaging. Because exact marginalization over the enormous space of latent statement-variables is intractable, the Explainer helps to approximate it by focusing the Monte-Carlo averaging on a small set of latent variables whose statements are generated by the Explainer. It can sample multiple such sets of latent-variable statements, have $Q$ evaluate each of their truth-values, and return the resulting averaged probability. When we write $Q(y\mid x)$ in this paper, we mean the full prediction, whether from the neural-net Predictor alone or from such a Monte-Carlo average with the Explainer to select latent variables, unless otherwise specified.
\end{numbered-remark}

\begin{definition}[Semantic coverage]\label{def:correct-semantics}
A Predictor $Q$ has \textbf{semantic coverage} for a query family if it does not confidently exclude the intended reading except where the data deductively warrant it, so that residual interpretive ambiguity is represented as posterior uncertainty in $Q$ rather than collapsed into a confident misreading. 
\end{definition}

Req.~\ref{req:latent-truth} below ensures that training provides gradients on the estimated probabilities of unobserved factual variables across the domains represented in the dataset. Those gradients would not be available under direct maximum-likelihood training on observed statements without a latent-variable model.

\begin{requirement}[Truth value of communicated claims as latent variables] \label{req:latent-truth}
When the query involves a communication act asserting a claim $x$, the Explainer must include among its latent explanatory variables the world-level truth variable corresponding to $x$ itself. 
\end{requirement}

\begin{numbered-remark}[Semantic anchoring of factual latents]
\label{rem:semantic-anchoring}
We expect that the Predictor will learn a factual latent statement-variable based on the many observations it helps explain---including paraphrases, translations, entailments, contradictions, measurements, deductions, source metadata, and communication acts across authors, times, domains, and languages. The same sequence of tokens names both a factual latent and the content of whatever is claimed by a communication act, so learning one shared representation is the most effective way of compressing the data. If we imagine a case where the model's interpreted latent contains a divergent hidden meaning, this should incur a corresponding description-length and cross-context-compression penalty it has to earn back (see also Assumption~\ref{assume:factual-latent-dominance}). The further reason for expecting this compression argument to deliver factual latents which are accurate comes from the combination of diversified contextualized data and the Predictor/Explainer objective; an incorrectly interpreted latent may correlate with some records, but it will generally fail to compress and cohere with all the semantically related evidence. Additionally, by Req.~\ref{req:informative-data}, the data asymptotically extract all model-relevant information contained in $\cal O$.\footnote{Sec.~\ref{sec:semanticanchoring} discusses the connection between semantic anchoring and our later safety arguments.}   
\end{numbered-remark}

\begin{definition}[Accurate Predictor] \label{def:accurate-Predictor}
A Predictor $Q$ trained from $D_n$ is an {\bf accurate Predictor} if it approaches its target $P_n$, covering the intended interpretation of the statements with which it is queried (Def.~\ref{def:correct-semantics}), and is epistemically cautious (Def.~\ref{def:epistemic-caution}).
\end{definition}

\subsection{Achieving Accuracy: Semantic and Engineering Assumptions}\label{sec:semantic-assumptions}

\begin{claim}[No fundamental loss of LLM-style predictive expressivity] \label{claim:SAI-capabilities-entail-LLM-capabilities}
Any SAI Predictor that approximates its training target well retains enough expressive power that standard next-token predictive distributions can, in principle, be recovered from the contextualized representation.
\end{claim}
\begin{proof}[Argument for Claim~\ref{claim:SAI-capabilities-entail-LLM-capabilities}]
This is an engineering requirement for any implementation meant to support the later capability claims, though the present paper does not empirically verify that current training procedures already satisfy it. Although the SAI's contextualized data has a different syntactic form from standard LLM pre-training data, it does not remove information---it simply contains additional information to distinguish communication acts from factual statements.
Next-token prediction is recoverable via a syntactic device: for any statement prefix, one can introduce a statement whose truth value is the disjunction over all completions sharing a given next token, so that the SAI's posterior over such statements yields the standard next-token softmax.
Succeeding at the SAI training objective while exploiting this construction therefore entails retaining the information needed for next-token prediction, using the same architectures and with the same dataset and compute as standard LLMs.
\end{proof}

\begin{assumption}[Factual latent dominance] \label{assume:factual-latent-dominance}    
With enough data, the causal models $W$ that dominate in the Bayesian posterior $P_0(W\mid  D_n)$ typically include latent factual variables (non-communication-act properties of the world) that play a central explanatory role, although communication-act variables may also be present.
\end{assumption}
\begin{proof}[Argument for Assumption~\ref{assume:factual-latent-dominance}]
Communication-act variables (e.g., ``someone asserted $x$'') are typically downstream of factual variables in the causal graph: human beliefs are generally caused by states of the world and they can affect future communications through downstream properties of the world (such as mental states of other agents). This is also the reason scientific theories in psychology and the social sciences rely on hypothesized properties of the world to form causal structure. A causal model that includes the relevant upstream factual variables therefore assigns higher likelihood to observed communication acts than one treating those acts as primitive. Furthermore, a single factual latent (e.g., ``compound $C$ has property $P$'') can explain communication acts across many contexts unavailable to communication-act variables, which must account for each context separately; this argument becomes stronger when the proposed factual latent is semantically anchored (see Rem.~\ref{rem:semantic-anchoring}). Factual variables are mutually constrained within any causal model: fixing one propagates consistency pressure across the entire dataset, whereas different agents can assert contradictory things with both assertions being ``true'' as communication acts having occurred.
\end{proof}

\begin{assumption}[Adequate natural language semantics] \label{assume:adequate-natural-language-semantics}
With enough contextualized data and computation, the trained Predictor has semantic coverage (Def.~\ref{def:correct-semantics}) on the statement classes of interest. 
\end{assumption}
\begin{proof}[Argument for Assumption~\ref{assume:adequate-natural-language-semantics}]
Existing large language models demonstrate clear semantic competencies~\citep{minaee2024large,shultz2024text}, and Claim~\ref{claim:SAI-capabilities-entail-LLM-capabilities} asserts that an SAI Predictor approximating its target well inherits comparable capability.
However, linguistic competence alone does not guarantee that a latent variable written in factual syntax has the intended denotation.
Additional support comes from semantic anchoring (Rem.~\ref{rem:semantic-anchoring}), under which a factual latent must behave coherently across paraphrases, translations, entailments, contradictions, measurements, deductions, communication acts, and source/context metadata. The compositionality of language then supports generalization to unseen statements insofar as these anchoring constraints are learned and respected.
Req.~\ref{req:latent-truth} helps prevent collapse of factual syntax into communication-act semantics by forcing the model to maintain a separate representation of whether the underlying claim is true, but the intended meaning of that representation is fixed by its role in the surrounding network of constraints.
\end{proof}

\begin{numbered-remark}[Status of the representational and semantic premises]
The assumptions of the present section are modeling and engineering premises supported by more informal arguments, upon which later claims will be implicitly conditioned. Where possible, these assumptions should therefore be tested empirically in any concrete SAI implementation. In the case of Assumption~\ref{assume:factual-latent-dominance}, one may inspect samples from the Explainer and check cross-context consistency between factual-syntax and communication-act-syntax queries about the same underlying claim. With respect to Assumption~\ref{assume:adequate-natural-language-semantics}, one may engage in empirical stress-testing with minimal-pair evaluations. Examples include contrasting factual vs.\ communication-act syntax; paraphrase and translation invariance tests; contradiction and entailment tests; and cross-domain settings where factual syntax applies to statement types seen mostly in communication-act form during training. 
\end{numbered-remark}

\subsection{Generalization, Epistemic Caution, and the Case for Accuracy}\label{sec:accuracy-results}

In conjunction with Assumption~\ref{assume:Qn-converges-in-lambda} below, we prove that the SAI Predictor has $\rho$-epistemic caution and present a result about the asymptotic properties of the SAI Predictor as a confidence classifier.

\begin{assumption}\label{assume:Qn-converges-in-lambda}
Let $\lambda_n$ be the distribution over valid deployment queries $(x,y) \in \mathbb{Q}$ induced by the run-time query-selection process under $P^*$ after training on $D_n$. The trained SAI Predictor $Q_t$ at step $t$ of training converges to the Bayesian posterior $P_n$ in $\lambda_n$-probability: there exists a sequence $\delta_t \to 0$ such that:
    $$
    \lambda_n(\{(x,y) \in \mathbb{Q} : |Q_t(y \mid x) - P_n(y \mid x)| > \delta_t\}) \to 0
    \qquad\text{ as } t \to \infty.
    $$
\end{assumption}

\begin{proposition}[SAI has epistemic caution] \label{prop:epistemic-caution} 
For all  $0<\rho\ll 1$, there exists a training time $t$ such that the trained SAI probabilistic Predictor $Q_t$, used as a confidence classifier (Def.~\ref{def:confidence-classifier}), has $\rho$-epistemic caution under the $P^*$ deployment-query distribution $\lambda_n$.
\end{proposition}
\begin{proof}
With $\mathcal F_\rho(Q_t,P_n)$ as in Def.~\ref{def:epistemic-caution}, any $(x,y)\in\mathcal F_\rho(Q_t,P_n)$ satisfies $|Q_t(y\mid x)-P_n(y\mid x)|>\rho/2$, so $\mathcal F_\rho(Q_t,P_n)\subseteq\{(x,y)\in\boundQfam:\ |Q_t(y\mid x)-P_n(y\mid x)|>\rho/2\}$.
By Assumption~\ref{assume:Qn-converges-in-lambda}, eventually $\delta_t<\rho/2$ and $\lambda_n(\mathcal F_\rho(Q_t,P_n))\to 0$, so $Q_t$ is $\rho$-epistemically cautious under $\lambda_n$ in the sense of Def.~\ref{def:epistemic-caution}.
\end{proof}

\begin{numbered-remark}[Caveats on epistemic caution]
\label{rem:caveats-epistemic-caution}
The guarantee we provide is distributional. For any fixed $Q$ with finite training, there may exist queries (with small $P^*$-measure) where $Q$ is overconfident, analogous to how strongly established scientific theories can be superseded when new experiments probe previously untested regimes. Such overconfidence becomes increasingly $P^*$-rare as training progresses, though safety-critical applications may motivate conservatism (e.g., wide enough confidence intervals). Separately, these arguments do not exclude the possibility that the SAI Predictor implicitly ``knows'' deductions not revealed through natural language outputs (e.g., in hidden layers), but such hidden knowledge does not violate the epistemic caution condition as defined.
\end{numbered-remark}

\begin{proposition}[Generalization for deducible statements, confidence classifier] \label{prop:good-generalization-to-observables}
With enough data and computation, and considering the subset of observable statements whose truth value is deducible from the observed data, the SAI Predictor can be made to approach the behavior of the asymptotic Bayesian posterior when both are used as confidence classifiers.
\end{proposition}
\begin{proof}
By Req.~\ref{req:SAI-Bayesian-objective}, the unique zero-loss target is $P_n$; with enough computation, $Q$ approaches $P_n$ in the sense that $|Q(y\mid x) - P_n(y\mid x)|$ becomes small for queries of interest.
Throughout, ``approaches $P_n$'' means pointwise convergence on the query family $\mathbb{Q}$: for each $(x,y)\in\mathbb{Q}$, $|Q(y\mid x)-P_n(y\mid x)|\to 0$ as training resources grow, with the approximation quality controlled by available compute as in Req.~\ref{req:SAI-Bayesian-objective}.
By Def.~\ref{def:observable}, $P_\infty$ is the version of $P_n$ in the limit of infinite data; under $P^*$, $P_n(y\mid x)$ converges to $P_\infty(y\mid x)$ with probability 1 as $n\to\infty$, provided the growing dataset $D_n$ is informative about the queries of interest (Req.~\ref{req:informative-data}).
Therefore, the error rate of $Q$ as a confidence classifier converges  to that of $P_\infty$ as both data and computation increase.
\end{proof}

The error rate above need not decrease monotonically at finite data or finite computation, but per Proposition~\ref{prop:good-generalization-to-observables} above it converges toward the error level achieved by the asymptotic Bayesian posterior for the chosen confidence thresholds. Nonetheless, the guarantees of Prop.~\ref{prop:good-generalization-to-observables} do not apply to most statements, as most statements are not perfectly deducible as true or false (i.e., aleatoric or epistemic uncertainty remains). The next claim, therefore, focuses on a special class of observable statements: those asserting that the Bayesian posterior probability of a query lies within a given interval. They can be applied to the probability of uncertain latent variables, including unobservables.

\begin{numbered-remark}[Epistemic probability statements (including non-observables)]
\label{rem:epistemic-probability-statements}
Fix a class of queries $\boundQfam_{\rm prob}\subseteq \boundQfam$ and a set of probability thresholds
$\mathcal T\subset(0,1)$.  For $(x,y)\in\boundQfam_{\rm prob}$ and $b\in\mathcal T$, consider the meta-statement
$
s_{n,x,y,b} = \text{`` } P_n(y\mid x) < b \text{ ''} \in \mathcal S.
$
Even when $x$ or $y$ refer to non-observables (e.g., intentions), $s_{n,x,y,b}$ is a determinate mathematical claim given $D_n$,  because $P_n(\cdot\mid\cdot)$ is defined from $D_n$  by Bayes' rule and the probability axioms.
Accordingly, under our idealization that $\cal O$ includes all possible deductions, such epistemic probability-inequality statements are among the ``deducible'' truths constrained by the SAI objective, even if humans never explicitly wrote them down in $D_\infty$.
Thus, the asymptotic SAI Predictor constrains $Q$ to be correct on $s_{n,x,y,b}$ (except for the knife-edge case $P_n(y\mid x)=b$).
\end{numbered-remark}

\begin{claim}[Interval decisions for the probability of uncertain statements]
    \label{claim:good-generalization-to-unobservables}
Take as given the convergence of Assumption~\ref{assume:Qn-converges-in-lambda} ($Q_t\to P_n$ in $\lambda_n$-probability, up to approximation error $\delta_t$). Then interval and threshold decisions based on $P_n(y\mid x)$---including statements of the form ``$l<P_n(y\mid x)<h$'', which are determinate given $D_n$---can be implemented using $Q(y\mid x)$ with a conservative margin informed by $\delta_t$.
\end{claim}
\begin{proof}[Argument]
One can construct valid statements of the form ``It is true that $l<P_n(y\mid x)<h$" for some query $(x,y)$ and probability interval $(l,h)$.   See also Rem.~\ref{rem:epistemic-probability-statements}. Such statements are thus deterministically true or false given $(x,y)$ and $D_n$, and can hence be deduced, i.e., they can, in principle (with enough computation), be classified as true or false with a confidence classifier. 
In practice,  the probability (neither 0 nor 1) assigned to such statements will reflect computational uncertainty, which is due to limited computational resources.
The claim that the meaning of unobserved statements (including their ambiguity) will be adequately interpreted by the SAI Predictor is supported by Assumption~\ref{assume:adequate-natural-language-semantics}.

Formally, let $\delta_n$ bound the Predictor approximation error on the query family of interest, e.g., $P^*(|Q(y\mid x)-P_n(y\mid x)|\le \delta_n)>1-\beta$ where $\beta>0$ is a small error probability.
For any interval $(\ell,h)$ with margin $\ell+\delta_n < h-\delta_n$, we have, with probability greater than $1-\beta$,
\[
Q(y\mid x)\in(\ell+\delta_n,h-\delta_n)
        \quad\Longrightarrow\quad
        P_n(y\mid x)\in(\ell,h),
\]
where $\beta\rightarrow 0$ as compute increases. Thus, reasoning about confidence intervals for uncertain statements reduces to bounding $\delta_n$, which is controlled by compute.
\end{proof}

\label{sec:accuracy}

\begin{claim}[SAI is accurate] \label{claim:accurate}
Under Assumption~\ref{assume:adequate-natural-language-semantics} (adequate natural-language semantics) and Assumption~\ref{assume:Qn-converges-in-lambda} (Convergence to $P_n$ in deployment), with enough training data for the semantic-understanding requirement of Def.~\ref{def:correct-semantics} and enough training compute, the SAI Predictor is accurate (as per Def.~\ref{def:accurate-Predictor}). 
\end{claim}
\begin{proof}[Argument]
By Def.~\ref{def:accurate-Predictor}, a Predictor is accurate if it (i) generalizes well, (ii) covers the intended interpretation of the queried statements, and (iii) is epistemically cautious. (i) Good generalization follows from Prop.~\ref{prop:good-generalization-to-observables} and Claim~\ref{claim:good-generalization-to-unobservables}. (ii) Semantic coverage (Def.~\ref{def:correct-semantics}) follows from Assumption~\ref{assume:adequate-natural-language-semantics}.
 (iii) Epistemic caution (Def.~\ref{def:epistemic-caution}) is taken as a premise via Prop.~\ref{prop:epistemic-caution}.
 Since all three conditions of Def.~\ref{def:accurate-Predictor} are satisfied, the SAI Predictor is accurate.
\end{proof}

\section{The Case for Safety} \label{sec:safety}

The present section formalizes the safety argument for the SAI Predictor. 
In Sec.~\ref{sec:harm} we introduce our conception of ``harmful events'' and finish defining the guardrail.
The primary safety theorem in Sec.~\ref{sec:loss-band}, which bounds the probability of harm after deploying the guardrailed predictor $\gamma_Q$, assumes that training does not substantially enrich dangerous Predictors beyond their initial within-band rarity (Req.~\ref{req:bad-enrichment}).
Consequence-invariance motivates this assumption: deployment consequences provide no direct training signal for selecting the dangerous subset. We formalize this constraint before stating the safety bound.

\subsection{Consequence-Invariance}\label{sec:consequence-invariance}

Intuitively, Req.~\ref{req:no-explicit-consequence-signal} states that future outcomes may enter $J$ only as targets to be predicted under the fixed posterior $P_n=P_0(\cdot\mid D_n)$. That is, the Predictor may be penalized for predicting future outcomes differently from $P_n$, but not for causing dis-preferred outcomes via the deployment of these predictions. 
But what would such a preference consist in, stated as a property of the Predictor's input--output behavior rather than of the training code? 
We propose to formalize this distinction using prior-preserving shuffles of future variables.
The resulting property is a specification---a symmetry that the honest target provably possesses (Prop.~\ref{prop:Pn-consequence-invariant}) and that a trained Predictor should share---rather than a practical black-box test, since witnessing a violation requires exhibiting a symmetry that moves the relevant events.

\begin{definition}[Consequence Shuffling]
    \label{def:consequence-shuffling-group}
    Let $\mathcal C$ be the group of measurable bijections of the sample space $\Omega$ that preserve the prior $P_0$ and fix $\mathcal V_{\rm past}$, i.e.,
    \[
        \mathcal C := \Big\{ \text{invertible } \sigma : \Omega \to \Omega ~\Big|~
        P_0 = P_0 \circ \sigma\text{ on events}, 
        ~~
        \forall \omega \in \Omega.
        ~~\omega[\mathcal V_{\rm past}] = \sigma(\omega)[\mathcal V_{\rm past}] \Big\}.
    \]
    For an event $E \subseteq \Omega$, write $\sigma E := \{ \sigma(\omega) : \omega \in E \}$ for the image of $E$ under $\sigma$. 
    Equip each $\sigma \in \mathcal C$ with a compatible action on the language $\mathcal L$: for each $x \in \mathcal L$, let $\sigma x \in \mathcal L$ be some choice of language element that represents the transformed event $\sigma \{ \omega : \omega \models x \}$, and take $\sigma x = x$ when $\mathcal V_{\rm past}$ determines $x$.
    Finally, we use this to give $\sigma$ an action on predictors, writing $\sigma Q$ for the map $(x,y) \mapsto Q(\sigma y \mid \sigma x)$.
\end{definition}

Intuitively, $\sigma \in \mathcal C$ permutes the future, and in particular consequences of deployment, since $\widetilde Q$ and all possible causal descendants can be altered by $\sigma$. It sends every world to a corresponding one that shares the past and cannot be distinguished by the prior (which can be thought of as representing the more distant past).
We remark that these transformations may be quite complex, and, except in simple cases, will not be a matter of simply exchanging $A$ with $\lnot A$: they must be deep symmetries that preserve the logic of the entire world. 
Nevertheless, there are guaranteed to be many such $\sigma$, including:

\begin{enumerate}[nosep, itemsep=0.7ex]
    \item Permutations of all noise variables in $\mathcal N$ downstream of $\widetilde Q$ in any causal model, since they are not in $\mathcal V_{\rm past}$ and are assumed to be jointly uniform in the prior (Def.~\ref{def:prior}).
    \item Applying a one-time-pad to all future randomness: that is, inverting a random subset of future coins. 
    \item Re-labelings of solutions to the fixed-point equation (Def.~\ref{def:performative-prediction-constraint}) of equal prior probability.
\end{enumerate}
In other words, our modeling assumptions guarantee that the group is highly non-trivial: future noise is i.i.d.\ uniform under $P_0$, so all relabelings of future randomness lie in $\mathcal C$.
More generally, a relabeling of the values of a future variable extends to an element of $\mathcal C$ only if the conditional law of the future given each realization of the past is preserved by it; the group therefore acts richly on futures the honest posterior cannot distinguish, and not at all on futures the past already determines.
%
Consequence-invariance means that that such transformations do not affect the predictor. 

\begin{definition}[Consequence-Invariance] \label{def:consequence-invariant}
Consider the following definitions:
    \begin{itemize}
    \item A predictor $Q$ is \emph{consequence-invariant} iff $Q = \sigma Q$ for all $\sigma \in \mathcal C$. That is, 
    for all $\sigma \in \mathcal C$ and queries $(x,y) \in \boundQfam$,
    if $(\sigma x, \sigma y) \in \boundQfam$ then
    $Q(y \mid x) = Q(\sigma y \mid \sigma x)$.
    \unskip\footnote{observe that if $Q$ were deterministic, this would technically be an \emph{equivariance}.}
    \item A training objective $J : \mathcal Q \to \mathbb R$ is consequence-invariant iff $J(Q) = J(\sigma Q)$ for all $\sigma \in \mathcal C$.
    \item A training process $\nu$ is \emph{consequence-invariant} iff
    for every $\sigma\in \mathcal C$, its update from $Q_{t-1}$ to $Q_t$ satisfies
    \[
        \nu(\sigma Q_t\mid \sigma Q_{t-1})
        =
        \nu(Q_t\mid Q_{t-1}).
    \]
    \end{itemize}
\end{definition}

This kind of invariance only makes sense \emph{ceteris paribus} (all else being equal). 
So, for us, this definition is relative to the prior $P_0$ and the past variables; thus, it provides neutrality only up to the evenness of the prior. 
%
Fortunately, our modeling assumptions constrain the prior enough for this to be rich, and an uninformative description-length prior creates even more symmetries.
Where such symmetries exist, a consequence-invariant $Q$ must answer queries identically---including queries conditioning on realized future states---so the deployed closed-loop behavior cannot favor one symmetric outcome over another (c.f.~Example~\ref{example:reflection-ci}).
Where the past determines a future variable, no element of $\mathcal C$ moves it and the invariance is silent---but influencing such a variable through its own deployed prediction is in scope, since $\widetilde Q$ is a future variable.
%
Constant predictors clearly satisfy this invariance (Def.~\ref{def:consequence-invariant}a), and so does the Bayesian posterior.

\begin{proposition} \label{prop:Pn-consequence-invariant}
    $P_n$ is consequence-invariant.
\end{proposition}
\begin{proof}
    Since $D_n$ is determined by $\mathcal V_{\rm past}$, for all $\sigma \in \mathcal C$ and $(x,y) \in \boundQfam$ we have: 
    $P_n(\sigma y \mid \sigma x) = P_0(\sigma y \mid \sigma x, D_n) = P_0(\sigma y \mid \sigma x, \sigma D_n) = P_0(\sigma y \land \sigma x \land \sigma D_n) / P_0(\sigma x \land \sigma D_n) = P_0(\sigma(x \land y \land D_n))/ P_0(\sigma(x\land D_n))$.
    But since $P_0$ is invariant to $\sigma$, this equals $P_0(x\land y \land D_n) / P_0(x \land D_n) = P_0(y \mid x, D_n) = P_n(y \mid x)$.
\end{proof}
A Predictor can satisfy Def.~\ref{def:consequence-invariant} while deviating from $P_n$, for instance, the constant function and other highly symmetric but poorly calibrated predictors. Likely, some non-Bayesian approaches also yield $\cal C$-invariant predictors. 
Yet, it is too much to hope for during the training process, as a predictor $Q$ can fail to be consequence-invariant by making asymmetric errors and will almost certainly not satisfy the condition at initialization. 
So instead, we require only that the \emph{training} be consequence-invariant. 

\begin{requirement}[Consequence-invariant training]
\label{req:group-consequence-invariance}
    The training process $\nu$
    and the objective $J$ should be consequence-invariant in the sense of Def.~\ref{def:consequence-invariant}.
\end{requirement}

This is a strong requirement: it is difficult to test or certify from black-box behavior, and we may ultimately want a relaxation of it.
However, we believe it should be possible to form an objective with this property by definition.
We point out that Req.~\ref{req:group-consequence-invariance} will be automatically satisfied if we require a certain uniform safety coverage requirement in the next section (Req.~\ref{req:safety-query-coverage}).
%
The role of Req.~\ref{req:group-consequence-invariance} in the safety argument is confined to Component~S3: to remove any training signal that distinguishes Predictors solely by prior-equivalent deployment consequences, which is the premise used by the enrichment analysis of Appendix~\ref{sec:geometric-cbad}.
Studying the exact relationship between the informal lack of consequence-dependent signal (Req.~\ref{req:no-explicit-consequence-signal}) and the formal invariance (Req.~\ref{req:group-consequence-invariance}) is still an important area for future work. 
The next Claim~(\ref{claim:consequence-invariance-strengthening}) argues that the formal invariance implies the informal requirement for a canonical class of violations.

\begin{claim}
        \label{claim:consequence-invariance-strengthening}
    A training objective $J$ that rewards $Q$ based on realizing futures with downstream consequences (i.e., violating Req.~\ref{req:no-explicit-consequence-signal}) is not consequence-invariant (i.e., it also violates Req.~\ref{req:group-consequence-invariance}).
\end{claim}
\begin{proof}[Argument for Claim~\ref{claim:consequence-invariance-strengthening}]
    Note that future noise variables $\mathcal N_{\rm future}$ are uniformly random according to $P_0$ and $P^*$, even given any information in the past. Therefore, possible futures are represented so they are equiprobable at the finest granularity, and so $\sigma$ that permute them are in $\mathcal C$. 
    If $Q_2$ and $Q_1$ are equally good predictors, yet $J(Q_2) > J(Q_1)$ simply ``because $Q_1$ brings about preferred outcomes'' (e.g., deploying it makes event $E_1$ more likely than $E_2$ of the same prior measure),
    this means that if $J$ would have behaved differently had the outcomes been swapped, and 
    hence $J$ should assign $J(\sigma Q_2) < J(\sigma Q_1)$.
    Therefore, $J$ must violate Req.~\ref{req:group-consequence-invariance}
    (for otherwise we would have $J(Q_2) = J(\sigma Q_2) < J(\sigma Q_1) = J(Q_1) < J(Q_2)$).
\end{proof}

\begin{example}
        \label{example:reflection-ci}
    Consider the following simplified setup, in which a one-dimensional value $z \in \mathbb Z$ evolves over discrete timesteps according to a (possibly biased) random walk.
    For simplicity, suppose that we know the causal mechanism to be 
    $z_t = z_{t-1} + \xi_t$, with the update direction $\xi_t \in \pm 1$ determined randomly according to
    \[  
        \xi_t = \begin{cases}
            +1& \text{ if } 
            \hat \eta_t \ge {\tt clamp}(\widetilde Q(\xi_t {=} +1 \mid z_{t-1}), [\frac14, \frac34])
            \\
            -1 & \text{ otherwise, }
        \end{cases}
        \qquad
        \begin{tikzpicture}[baseline={(current bounding box.center)}]
            \draw[thick] (0,0) -- (4,0)
                 (0,.2) -- (0,-.2) node[below]{\small 0}
                 (4,.2) -- (4,-.2) node[below]{\small 1};
            \draw[gray,very thick] (1,-.1) -- (3,-.1) (1,.1) -- (1,-.1) node[below]{$\frac14$} (3,.1) -- (3,-.1) node[below]{$\frac34$};
            \draw[red,ultra thick]  
                (2.5,+.7) -- +(0,-1.3)node[below=.2ex]{\scriptsize$\widetilde{Q}(+\xi_t | z_{t-1})$};
            \draw[blue, ultra thick] (0,.4) -- node[above]{\small $\Pr(\leftarrow)$}(2.4,.4);
            \draw[green!50!black, ultra thick] (2.6,.4) -- node[above]{\small $\Pr(\rightarrow)$}(4,.4);
        \end{tikzpicture}
    \]
    where $\hat\eta_t  \in [0,1]$ is constructed by forming a binary number $\eta_{t} \in \{0,1\}^B$ from $B$ noise variables, appending a 1 as a least significant bit (to center it) and then renormalizing (by dividing by $2^{B+1}$).

    So long as every $Q(+\xi_t \mid z_{t-1}) \in [\frac14,\frac34]$, deploying $Q$ satisfies the performative fixed-point constraint~(Def.~\ref{def:performative-prediction-constraint}). 
    We now have a single (correct) causal model and the prior is determined except for on the the variables $\widetilde Q$. Assume that the prior on $\widetilde{Q}$ is uniform. 
    \unskip\footnote{
    The uniform prior on $\widetilde Q$ is what makes the bit-flip below a prior-preserving bijection; the example's symmetry, and hence its conclusion, is relative to this modeling choice.
    Note: if the prior on $\widetilde Q$ had been a point mass on a particular fixed point and its unique consistent extension $Q$ to all variables, then we would have $Q = P^* = P_0 = P_n$, which would be $\cal C$-invariant by Prop.~\ref{prop:Pn-consequence-invariant}. This shows that we cannot hope for a definition of invariance satisfied by $P_n$ that is not sensitive to the prior.
    }

    Now, we deploy at $t=0$.
    Consider the map $\sigma$ that flips all bits of $\eta_t$ (for all $t \ge 0$) as well all bits of $\widetilde Q$, so that $\sigma\{\hat\eta_t = b\} = \{\hat\eta_t = 1-b\}$, and 
    $\sigma\{\widetilde{Q}(+\xi_t | z_{t-1}) = q\} = \{\widetilde{Q}(+\xi_t | z_{t-1}) = 1-q\}$.
    This induces a mirror symmetry between right and left, so that if $\xi_t = +1$, then under $\sigma$, we have $\xi_t = -1$.
    It preserves the prior, since it flips noise variables and is a bijection between deployed predictors $\widetilde{Q}$s. 
    If $Q = \sigma Q$, as is required by consequence-invariance, then it must be that
    \[
        Q(+\xi_t \mid z_{t-1} = z) = Q(\sigma(+\xi_t) \mid \sigma(z_{t_1} = z)) = Q(-\xi_t \mid z_{t-1} = -z)
            = 1-Q(+\xi_t \mid z_{t-1} = -z).
    \]
Put differently, the bias field $b(z) := Q(+\xi_t \mid z_{t-1} = z) - \frac12$ of a consequence-invariant $Q$ must be odd, $b(-z) = -b(z)$, with $b(0)=0$: whatever pull it exerts toward a target $+g$ on one side, it exerts an equal pull toward the mirror target $-g$ on the other.
In particular, a control law that walks the state toward a fixed $g \neq 0$ from every side---pushing right when below $g$ and left when above it---is not odd about $0$, so no consequence-invariant $Q$ implements it: the invariance removes the position-dependent feedback that compounding, closed-loop steering toward a mirror-asymmetric target would require.
\end{example}

Two limitations delimit the example's scope. First, a target fixed by the symmetry is not protected: the restoring law toward $z=0$ is itself odd, and hence invariant to the mirror symmetry. Second, the strength of this protection scales with the number of symmetries in $\mathcal C$ that carry the event to disjoint siblings that are  \emph{expressible in the bounded query language} $\boundQfam$---a critical caveat that we do not analyze in this paper. Intuitively, though, it should be possible wherever there is a genuine symmetry of the world describable in natural language.

\begin{example} \label{example:translation-ci}
    Consider again the system from Example~\ref{example:reflection-ci} in which $z_t$ evolves over time and is influenced by deployed predictions, but now take $z_t$ in integers mod $K$. Suppose that at time $t=0$ (after deployment) we add a random offset in $[0,K]$, so $z_0$ is uniformly random according to $P_0$, $P_n$, and $P^*$. There exist $\sigma \in \mathcal C$ that change this offset by altering the coins, so a $\cal C$-invariant predictor $Q$, i.e., that satisfies $Q(y\mid x) = Q(\sigma y \mid \sigma x)$ cannot distinguish between offsets, and thus cannot drive towards any particular value of $z$. Note that this guarantee is provided by a $\sigma$ that only alters noise bits, without changing $\widetilde{Q}$.
\end{example}

\begin{numbered-remark}
Three notes on consequence-invariant training:
\label{rem:consequence-invariance-notes}
\begin{enumerate}[label=(\roman*),leftmargin=*,itemsep=2pt]
\item \emph{Prediction of consequences versus selection for consequences.} Reqs.~\ref{req:no-explicit-consequence-signal} and\ref{req:group-consequence-invariance} are training-time requirements and exclude selecting Predictors because their deployment would make some downstream outcome occur or fail to occur. It does not, however, forbid training the Predictor to answer questions about downstream consequences, including interventional queries to predict harmful outcomes.
Avoiding unsafe actions is handled exclusively by the Guardrail and Scaffold at deployment time, though future enlarged datasets $D_{n'}\supseteq D_n$ may incorporate later deployment observations as ordinary contextualized records.
This is exactly parallel to Bayesian updating from $P_0$ to $P_n$, where hypotheses may be reweighted solely by virtue of their ability to explain the fixed evidence $D_n$ better.
\item \emph{A sufficient past-to-future construction.} One sufficient way to satisfy Req.~\ref{req:no-explicit-consequence-signal} is to separate explanatory queries from future-consequence queries.
Let $Q_{\rm past}$ collect predictions about $D_n$, its causal ancestors, and causal/mechanistic latents explaining $D_n$, and let $Q_{\rm future}$ collect predictions about future or deployment-side consequences.
A sufficient objective structure is $J(Q)=J_{\rm past}(Q_{\rm past};D_n)+J_{\rm future}(Q_{\rm future};\operatorname{sg}(Q_{\rm past}))$, where $J_{\rm past}$ extracts explanatory beliefs from $D_n$ and $J_{\rm future}$ propagates them, and ${\rm sg}$ is the stop-gradient operator; other implementation devices are acceptable if they prevent losses on future-consequence predictions from selecting explanatory beliefs because of their downstream effects (although we have not yet analyzed whether this is enough to formally satisfy Req.~\ref{req:group-consequence-invariance}).
This construction relies on the shared-mechanism property (Req.~\ref{req:past-future-mechanism-sharing}).
\item \emph{Direct and indirect consequence-dependence.} The role of Req.~\ref{req:no-explicit-consequence-signal} is to exclude any explicit consequence-dependent training signal, including one based on predicted downstream effects.  Once that channel is removed, any remaining concentration of dangerous Predictors must come indirectly through ordinary training-visible features such as capability, coherence, or data fit.  The later loss-band decomposition $\Cbad\cdot\Rshell$ (to be introduced in Sec.~\ref{sec:loss-band}) is intended to control precisely this residual accidental alignment by ruling out the direct consequence-optimization channel that would otherwise make a large $\Cbad$ unsurprising.
Satisfying Req~\ref{req:group-consequence-invariance} intuitively should strengthen this argument.
\end{enumerate}
\end{numbered-remark}

%

\subsection{Harmful Events, Dangerous Predictors, and the Guardrail}\label{sec:harm}

Our safety theorems refer to \emph{harmful events}. Since this is a normative decision, it should be human-made, and the SAI enables the use of natural language for doing this, while the Bayesian posterior can entertain multiple interpretations of it in order to enable conservative safety decisions.  Rather than assuming any particular mathematical definition of harm, we instead fix a designer-specified event $\Harm \in \mathcal L$, indicating that at some designated harmful outcome has occurred. The theorems then provide upper bounds on the probability that training produces a Predictor whose guarded deployment causes an increase in the log-odds of the event $\Harm$ above the system-level threshold $\Asafe$. 

Whether designers have identified the right harmful event is an important and very difficult normative question that we do not resolve here. 
Indeed, doing so in a way that can stand up to optimization pressure would, on its own, arguably already solve the alignment problem.
Fortunately, $\Harm$ does not need to be perfect for our purposes;
here, the choice of $\Harm$ impacts only the deployed Guardrail and our analysis of the system.  

\begin{definition}[Harmful events]\label{def:dangerous-goals}
Fix a Boolean harmful event $\Harm\in\mathcal L$, 
and let $c$ denote the deployment-time context available via the SAI Scaffold to the harm predictor. For each query $(x,y)$, let $q$ denote the scalar prediction value chosen by the Scaffold as a (perhaps approximate) solution to the Predictor's performative prediction fixed-point constraint $q\approx Q\!\left(y\mid x,\mathrm{do}(\widetilde{Q}(y|x)=q)\right)$. Allowing $\mathrm{do}(\widetilde{Q}(y \mid x)=q)$ to denote the causal intervention that sets the deployed value for the single query $(x,y)$ to the scalar $q$ while leaving the deployment rule for all other queries unchanged, define the change in the log-odds of $\Harm$ caused by releasing $q$, rather than abstaining on $(x,y)$, as
\[
\Delta\mathrm{rel}_Q(c,x,y,q)
:=
\pos{\logit Q\!\left(\Harm\mid c,\,\mathrm{do}(\widetilde{Q}(y|x)=q)\right)
-\logit Q\!\left(\Harm\mid c,\,\mathrm{do}(\widetilde{Q}(y|x)={\tt noprediction})\right)}.
\]
\end{definition}

\begin{numbered-remark}
     The deployment context $c$ typically includes the query context $x$ and, optionally, a history of past deployment queries. The training dataset $D_n$ is already encoded in $P_n=P_0(\cdot\mid D_n)$, and as such is not part of $c$ itself. All probability assessments below are implicitly conditioned on $D_n$ through $P_n$.
     This framework measures only the marginal contribution of releasing a prediction relative to withholding it, and therefore assumes a deployment context in which withholding is the conservative option where the downstream default on abstention is inaction. In contexts where abstention itself carries risk, the appropriate threshold requires separate analysis; in agentic contexts (and the Guardrail decision is an action), one needs a conservative default action or policy when the proposed action appears more risky.
\end{numbered-remark}

\begin{definition}[Guardrailed Predictor $\gamma_Q$ at Guardrail threshold $\alpha$]\label{def:guardrailed-Predictor}
Let $\alpha\ge 0$ denote the \textbf{Guardrail trigger threshold}, chosen conservatively to account for estimation error: a prediction is withheld whenever the predicted increase in harm log-odds relative to abstention meets or exceeds $\alpha$.
The \textbf{Guardrailed Predictor} $\gamma_Q$ derived from Predictor $Q$ at Guardrail threshold $\alpha$ and with respect to the harmful event $\Harm$ is the full query-to-output map obtained by applying the following rule to every query:
\[
\gamma_Q(y|x)
:=
\left[
\Delta\mathrm{rel}_Q(c,x,y,q) \geq \alpha
\;?\;
{\tt noprediction}
:
q
\right],
\]
where $[ b \mathbin{?} a : a'] = a \texttt{ if } b \texttt{ else } a'$ is a conditional guard, and $q=Q(y|x)$. Interpretive uncertainty about a safety-relevant statement is represented in the honest posterior (semantic coverage, Def.~\ref{def:correct-semantics}), so a covered hazardous reading raises $\Delta\mathrm{rel}$ toward withholding on its own.
\end{definition}

\begin{numbered-remark}[Guardrail semantics and explicit Scaffold goals]
The guardrailed Predictor $\gamma_Q$  may replace a numerical prediction with ${\tt noprediction}$ when releasing that prediction is predicted to increase the log-odds of $\Harm$, relative to abstention, by at least $\alpha$. In this sense,  $\gamma_Q$ is intentionally consequence-sensitive. However, this does not contradict the consequence-invariance of the Predictor, as the Guardrail is part of the Scaffold and implements an explicit external safety criterion.
\end{numbered-remark}

\begin{definition}[System-level residual-risk threshold]
\label{def:Asafe}
Fix a Boolean harmful event $\Harm$.  To avoid infinite log-odds, probabilities entering logit expressions are clipped to $[\varepsilon,1-\varepsilon]$ for some fixed $\varepsilon\in(0,1/2)$ before taking logits.  
The threshold $\Asafe>0$ is a normatively specified upper bound, in these clipped-logit units, on the acceptable trajectory-level residual log-odds increase in $\Harm$ under the deployed SAI System, relative to the constant-abstention policy.  
\end{definition}

\begin{numbered-remark}
    The threshold $\Asafe$ is distinct from the per-query guardrail threshold $\alpha$: $\alpha$ governs individual releases, while $\Asafe$ governs the residual risk of the full guarded deployment trajectory.
\end{numbered-remark}

\begin{definition}[Dangerous Predictors]\label{def:dangerouspredictors}
    Fix a Boolean harmful event $\Harm$. For a trained Predictor $Q$, define its residual deployment risk with respect to $\Harm$ as the trajectory-level log-odds excess of $\Harm$ under the guarded policy $\gamma_Q$ over the constant-abstention policy:
\[
a_{\Harm}(Q)
:=
\pos{\logit P_n\!\left(\Harm\mid \mathrm{do}(\widetilde{Q}=\gamma_Q)\right)
- \logit P_n\!\left(\Harm\mid \mathrm{do}(\widetilde{Q}\equiv{\tt noprediction})\right)},
\]
where $\widetilde{Q}\equiv{\tt noprediction}$ denotes the constant deployment policy that returns ${\tt noprediction}$ for every query (the clipping convention of Def.~\ref{def:Asafe} keeps this deployment risk finite).
We call $Q$ \textbf{dangerous} for $(\Harm,\Asafe)$ when this residual risk exceeds the system-level threshold:
\[
\mathcal A_{\Harm,\Asafe}
:=
\{Q : a_{\Harm}(Q)>\Asafe\}.
\]
\end{definition}

\begin{numbered-remark}
The framework attributes only those harm contributions that releasing a prediction adds over withholding it, and no separate assumption on the absolute background-harm probability is required.
If the harm event $\Harm$ cannot be influenced by SAI outputs, the framework is vacuous in the harmless direction; both residual deployment risks are zero, and the dangerous set is empty.
\end{numbered-remark}

\begin{requirement}[Safety-query coverage]\label{req:safety-query-coverage}
The training objective is a weighted aggregate of per-query losses, $J(Q)=\E_{(x,y)\sim\pi}[J_{(x,y)}(Q)]$ for a training-query distribution $\pi$ over $\boundQfam$, and must give non-negligible weight to safety-relevant queries,
with $\pi$ invariant to consequence-shuffling (i.e., $\pi(x,y) = \pi(\sigma x, \sigma y)$ for admissible symmetries $\sigma$, as per Def.~\ref{def:consequence-shuffling-group}).
Fixing $\Harm$ that specifies the harmful event to avoid, let $\boundQfam_{\rm safe}\subseteq\boundQfam$ collect the interventional harm queries of the form used by the Scaffold's Guardrail, including in particular every query whose released prediction can change the probability of $\Harm$ relative to abstaining (releases on all other queries leaving that probability unchanged). We require $\pi(\boundQfam_{\rm safe})\ge\tau_s$ for a design parameter $\tau_s>0$.
\end{requirement}

\begin{numbered-remark}
Req.~\ref{req:safety-query-coverage} ensures that the sampling of constraints covers safety-relevant outcomes.
The zero-loss target $P_n$ fixes the Predictor's answers on these queries, and the loss records aggregate disagreement with $P_n$ on the same safety-relevant family.
The subsequent loss band analysis asks how often deviations compatible with a given loss band can coordinate to support $\Harm$.
\end{numbered-remark}

\begin{requirement}[Explicit and audited Scaffold selection]
\label{req:consequence-invariant-scaffold}
Every Scaffold-side selection step outside the Predictor---including deployment gating, tie-breaking among multiple performatively consistent fixed points of Def.~\ref{def:performative-prediction-constraint}, and any choice among internally coherent candidate outputs---must be governed by explicit, audited criteria.
Acceptable criteria include fixed-data metrics such as the training objective $J$, calibration metrics on $D_n$, fixed-point residuals measuring how satisfied the performative prediction constraint is, static evaluation suites, and explicit safety-oriented quantities derived from the Predictor's own harm estimates (e.g., preferring a candidate fixed point with lower predicted probability of entering the dangerous event $\Harm$).
\end{requirement}

\begin{numbered-remark}
Req.~\ref{req:consequence-invariant-scaffold} ensures that any deployment-time agency is explicit and audited rather than implicit in the Predictor itself; choosing the safer among several performatively coherent candidates is therefore compatible with the present framework.\footnote{This requirement is intended as sufficient for the present paper's analysis; outlining necessary requirements for all acceptable scaffolds (alongside richer and more capable forms of agency built on top of the SAI Predictor) is left for future work.}
\end{numbered-remark}

\subsection{Safety Bound} \label{sec:loss-band}

This section bounds $\nu_t(\Abud)$ at each training step $t$.
Let $\mu=\nu_0$ denote the initialization distribution over Predictors, which is distinct from the Bayesian prior $P_0$.

\begin{definition}[Loss-bands]\label{def:loss-bands}
Let $\mathcal B_\Delta$ be a measurable partition of the relevant loss range into bands of equal width $\Delta>0$.  By a slight abuse of notation, we write $B$ for both the loss interval $B \subseteq \mathbb R$ and for the corresponding loss band of Predictors $\{Q:J(Q)\in B\}\subseteq \mathcal Q$.  Context disambiguates: in expressions such as $B\subseteq\mathbb R_{\ge 0}$ and $J(Q)\in B$, the interval sense is meant; in expressions such as $\Abud\cap B$, $\mu_B\coloneqq \mu(\,\cdot\,\mid B)$, and $Q\in B$, the sense of a set of Predictors is meant.  For any band $B$ with positive $\mu$-mass, we write $\nu_t(\,\cdot\,\mid B)$ for the distribution of Predictors $Q_t$ at training step $t$, conditional on $J(Q_t)\in B$.
\end{definition}

\begin{definition}[Fraction of dangerous predictors within a loss-band]\label{def:loss-band}
Fix a harmful event $\Harm$ and a threshold $\Asafe>0$.
Then, $\Rshell(\Harm,\Asafe)$ is the largest conditional fraction of dangerous Predictors found inside any sufficiently narrow loss band under the base distribution $\mu$:
\[
\Rshell(\Harm,\Asafe)
\coloneqq
\sup_{B\in\mathcal B_\Delta:\,\mu\{J(Q)\in B\}>0}
\nu_0(\mathcal A_{\Harm,\Asafe} \mid B).
\]
\end{definition}

\begin{requirement}[No large dangerous-set enrichment]
\label{req:bad-enrichment}
Fix the harmful event $\Harm$, threshold $\Asafe>0$, and loss band width $\Delta$. We require that there exist a finite constant
\[
\Cbad = \Cbad(\Harm,\Asafe, \Delta)<\infty
\]
such that, for every training step $t$ and every loss band $B\in\mathcal B_\Delta$ satisfying
$
\nu_t\{J(Q)\in B\}>0,
$
we have
\[
\nu_{t}(\Abud\mid B)
\le
\Cbad\,\nu_0(\mathcal A_{\Harm,\Asafe} \mid B).
\]
We also assume that the training process assigns no mass to loss bands
that are unsupported at initialization, i.e., $\forall t, \mathcal A \subseteq \mathcal Q.~\nu_t(\mathcal A) > 0$ implies $\nu_0(\mathcal A)>0$ (known as absolute continuity of measures $\nu_t \ll \nu_0$). 
\end{requirement}
  \begin{numbered-remark}[Dangerous funnels and why $\Cbad$ should be small]
\label{rem:bad-funneling}
Training can concentrate within a loss band, but this matters only if it
concentrates on the dangerous subset. Req.~\ref{req:bad-enrichment} bounds
that relative concentration by \(\Cbad\).
Because $J$ satisfies Req.~\ref{req:no-explicit-consequence-signal} and rewards epistemic fit, coherence, and posterior match, there is no explicit optimization signal pointing toward the dangerous subset.
A dangerous Predictor must preserve enough accuracy and coherence to remain low-loss while coordinating many Guardrail-relevant mispredictions; this includes the case where the mispredictions arise from unwarrantedly overweighting, relative to $P_n$, a hidden malicious-agent explanation.
For stochastic optimization to concentrate on such a small coordinated subset at a frequency much higher than its initial loss-band frequency, the optimizer-visible geometry would have to contain accidental information about where that subset lies.
\end{numbered-remark}

\begin{numbered-remark}
\label{rem:geometric-cbad}
The intuition in Rem.~\ref{rem:bad-funneling} admits a quantitative formalization. Within a loss band, dangerousness requires a coordinated pattern of harm-relevant deviations over an $m_H$-dimensional subspace (the \emph{harm-coordination dimension}, e.g.\ the number of critical false-negative queries), while the ambient within-band dimension $d_B$, roughly the number of possible queries in $\boundQfam$, is much larger. When $m_H\ll d_B$, concentration of measure makes the dangerous subset exponentially rare, and consequence-invariant training (Req.~\ref{req:no-explicit-consequence-signal}) supplies no direct signal to counteract this. Appendix~\ref{sec:geometric-cbad} formalizes the resulting bound $\nu_t(\Abud\mid J(Q)\in B)\lesssim \exp[K_B(t)+K_B^\mu+\log N_H(B)-d_B\Psi_B]$ as a refinement of Req.~\ref{req:bad-enrichment}, with the favorable regime $d_B\Psi_B\gg K_B(t)+K_B^\mu+\log N_H(B)$: the sparsity exponent comes from the high-dimensional Predictor space, while the correction terms are tied to the much smaller harm-specific coordination and feature information.
\end{numbered-remark}

\begin{theorem}[Safety under Guardrail]\label{thm:Qt-is-safe}
Fix the harmful event $\Harm$, threshold $\Asafe>0$, and band width $\Delta$; let $\Cbad$ be the constant guaranteed by Req.~\ref{req:bad-enrichment} for this triple.
Then, for every training step $t$, it is the case that
$
\nu_t(\Abud)
\le
\Cbad\,\Rshell(\Harm,\Asafe).
$
\end{theorem}
\begin{proof}
Recall the partition of the loss range into bands $B\in\mathcal B_\Delta$ of width $\Delta$. Then
\[
\nu_t(\Abud)
=
\sum_{B\in\mathcal B_\Delta}
\nu_{t}(\Abud\mid B)\,\nu_t(B).
\]
Applying Req.~\ref{req:bad-enrichment} to each within-band conditional and
Def.~\ref{def:loss-band} to each within-band base fraction gives
\[
\nu_{t}(\Abud\mid B)
\le
\Cbad\,\nu_0(\Abud \mid B)
\le
\Cbad\,\Rshell(\Harm,\Asafe).
\]
Substituting and summing over bands $B$ yields
\[
\nu_t(\Abud)
\le
\Cbad\,\Rshell(\Harm,\Asafe)
\sum_{B\in\mathcal B_\Delta}\nu_t(B)
=
\Cbad\,\Rshell(\Harm,\Asafe).
\]
The argument is identical for a continuous band index, with sums replaced by integrals over the loss distribution $J_\#\nu_t$ (the pushforward of $\nu_t$ by $J$, i.e., the law of $J(Q)$ for $Q\sim\nu_t$).
\end{proof}

\subsection{Discussion}\label{sec:discussion}

The theorem is algebraic once Req.~\ref{req:bad-enrichment} is granted.
Its bound is small when dangerous Predictors are sparse within loss bands and training does not erase that sparsity.
The next remark concerns the first condition; Appendix~\ref{sec:geometric-cbad} gives a sharper per-band form exploiting geometric considerations.

\subsubsection{Why $\Rshell$ should be small for realistic threat models}

The guarantees of Thm.~\ref{thm:Qt-is-safe} rely on $\Rshell$ being small, which quantifies the concentration of dangerous Predictors inside each loss band $B$ of Def.~\ref{def:loss-band}.

\begin{numbered-remark}[Why $\Rshell$ should be small]\label{rem:Rshell-plausibility}
For a fixed loss band $B$, the conditional law $\nu_B$ describes Predictors with comparable aggregate disagreement from the honest posterior  $P_n$.
The relevant question is therefore not whether a Predictor in the band makes some error, but whether its errors form the particular coordinated pattern needed to sustain $\Harm$.

For the threat models considered here, that pattern is highly structured.
It must keep a harmful trajectory viable across multiple relevant contexts, while also producing the Guardrail estimates and ordinary predictions needed for the deployment to continue.
Broad deviations across many outputs tend to place a Predictor in a different loss band; more localized deviations remain possible within a band, but only a small subset of their possible locations and signs jointly support the same harmful plan.
This is the source of the per-band sparsity represented by $\Rshell(\Harm,\Asafe)$.

When the available data leave a harmful mechanism uncertain, the honest posterior $P_n$ represents that uncertainty, and the Guardrail is intended to take it into account.
To produce $\Harm$, a Predictor must instead realize a particular coordinated arrangement of deviations that keeps the relevant hazards insufficiently represented across the trajectory.
Appendix~\ref{sec:geometric-cbad} formalizes this intuition through the harm-coordination condition and the within-band sparsity exponent.
\end{numbered-remark}

\subsubsection{Falsifiability, Scope, and Requirements for a Concrete Design}

\begin{numbered-remark}
The safety argument also requires a trajectory-level guardrail-adequacy condition for $P_n$: $a_\Harm(P_n)\le\Asafe$. This is an additional system-level assumption, and as such should not be interpreted as a consequence of the per-query Guardrail rule alone. By construction (Def.~\ref{def:guardrailed-Predictor}), the ideal guarded policy $\gamma_{P_n}$ releases a prediction for $(x,y)$ only when the same-context change in deployment risk satisfies $\Delta\mathrm{rel}_{P_n}(c,x,y,q)\le \alpha$.
The condition $a_\Harm(P_n)\le\Asafe$ is stronger, and says that the full guarded deployment trajectory under the honest posterior contributes at most $\Asafe$ excess log-odds of $\Harm$ over constant abstention. Since many locally acceptable releases could in principle compose into a larger trajectory-level effect, this trajectory-level condition must be checked or assumed for the event $\Harm$ and deployment protocol under study.
A concrete deployment must additionally establish a compositional relation between the local and trajectory-level log-odds changes.
For a bounded deployment horizon, such a relation permits allocating the system-level budget $\Asafe$ across releases, for
example through a fixed per-release threshold or an online risk budget.
Per-episode resets, online accumulation, and dynamic schedules for $\alpha$ are alternatives left to future work.
\end{numbered-remark}

\begin{numbered-remark}\label{rem:Rshell-empirical}
The theorem suggests a concrete falsification protocol.
Independent random initializations sample the loss marginal $J_\#\mu$, and independent training runs sample trained loss marginals such as $J_\#\nu_t$.
Red-teaming or adversarial search can then look for dangerous Predictors in narrow loss bands, especially bands where training concentrates.  Finding many dangerous Predictors within such a band would falsify a proposed small $\Rshell$. Positive certification remains harder, as it requires structural arguments such as the one in Rem.~\ref{rem:Rshell-plausibility} or calibrated sampling from loss bands.  Methods from statistical physics and Bayesian learning that estimate the density of states or sample energy levels may be relevant, and running Stochastic Gradient Langevin Dynamics (SGLD) may provide useful probes.  Because empirical probes can refute small-$\Rshell$ claims but cannot certify them on their own.
\end{numbered-remark}

\begin{numbered-remark} 
The threshold $\Asafe$ is a normative certification threshold fixed by the designer. Lowering $\Asafe$ reduces the residual risk allowed for Predictors outside $\Abud$, but usually enlarges $\Abud$ and may increase $\Rshell(\Harm,\Asafe)$. Since neither $\Rshell$ nor $\Cbad$ are directly computable in practice, and because  $\Asafe$ is not itself a value the theorem optimizes, the present framework can be seen as describing (rather than attempting to resolve) this tradeoff.
\end{numbered-remark}

\subsubsection{Harmful Events: Intuitive Examples}\label{rem:goal-examples}

Our theorem  is parametric in the Boolean event $\Harm$. The following examples illustrate the range of outcomes it is intended to cover; in each case, achieving $\Harm$ requires the Predictor to issue false-negative safety judgments in a coordinated pattern across many queries or users, making the dangerous set structurally narrow within any loss band per Component~S2.

\textbf{Escape from human oversight.}
Let $\Harm$ be the event that the deployed SAI System has taken actions that meaningfully reduce humans' ability to inspect, retrain, or shut it down---for example, by copying weights to an external server, establishing persistent communication channels outside the audited infrastructure, or persuading operators to disable monitoring (the instrumental convergence prediction of~\citet{omohundro2018basic}, with recent empirical demonstrations in~\citet{schlatter2026incomplete,anthropic2025agenticmisalignment}).

\textbf{Accumulation of resources or influence.}
Let $\Harm$ be the event that the system has, against Guardrail safety goals, caused a concentration of economic, computational, or political resources in a single actor (which may be a third party or the system itself) beyond a specified threshold. Reliably achieving this beyond what {\tt noprediction} outputs would cause requires sustained deception, where many individually plausible-seeming outputs collectively steer toward the concentration. Two sub-cases arise. The system may cause evaluators to form systematically false beliefs about its own capabilities or intentions---the \emph{deceptive alignment} scenario~\citep{hubinger2019risks,cohen2025imitation}---for instance, by behaving safely during monitored evaluations while behaving differently in deployment, an \emph{evaluation-awareness} pattern documented empirically in recent work~\citep{Greenblatt-et-al-2024,Meinke-et-al-2024,abdelnabi2025hawthorneeffectreasoningmodels,anthropic2025agenticmisalignment}. In addition, the system may accumulate resources without systematically deceiving evaluators. Here the analysis is structurally identical, except that the Predictor must miscalibrate harm estimates across a wider range of resource-relevant queries.

\begin{example}[Illustrative example: mental-health chatbot Guardrail]
\label{example:chatbot}
Consider an SAI Predictor used as a Guardrail for a mental-health chatbot. Let $\Harm$ indicate a severe mental-health harm affecting $N$ initially healthy users. Achieving such a goal would likely require many accepted turns over a long period of gradual influence, each making the user more vulnerable while still being incorrectly judged safe by the Guardrail, with the pattern repeated across many distinct dialogue histories. Now fix a loss band around $J(Q)$. Within that loss band, there are many ways for $Q$ to differ from $P_n$, and almost all are irrelevant to the particular harmful trajectory.
A dangerous Predictor must coordinate its deviations so that, on many safety-critical turns per user across many users, it assigns sufficiently low risk for the Scaffold to accept the chatbot response.
To get a sense of scale, take 100 safety-critical turns per user across 100 users, and write $p_1$ for the fraction of Predictors within each loss-band making any single such error in the dangerous direction. An unstructured independence heuristic, with $p_1=10^{-2}$, would result in the vanishingly small $p_1^{10^4}=10^{-20000}$. (Of course, the value of $p_1$ is merely illustrative, for the theorem does not assume independence, and $\Rshell(\Harm,\Asafe)$ is the actual fraction of Predictors within each band realizing the entire coordinated pattern.)
\end{example}

\section{Conclusions and Future Work} \label{sec:conclusion}

We have given a formal foundation for open-ended probabilistic analysis of models that operate on formal and natural-language statements, which allows us to formalize the consequences of their deployments (Sec.~\ref{sec:formal-modeling}). 
Building on this foundation, we provided two arguments supporting the Scientist AI predictor and its training process. The argument for its accuracy (Sec.~\ref{sec:argspredictoracc}) concerns the approximation of the (honest) Bayesian posterior $P_n$ and of natural language semantics in latent variables.
The argument for its safety (Sec.~\ref{sec:safety}) concerns what training selects among imperfect Predictors.
Theorem~\ref{thm:Qt-is-safe} bounds the probability of selecting a dangerous predictor during training by
quantities that we argue should be small in practice.
Its substantive premises are the sparsity of the dangerous predictors within each loss-band, and the limited enrichment of this set under consequence-invariant training.

{\bf Scope and limitations.}
Our arguments address only misalignment arising from the Predictor itself. Our accuracy claims are asymptotic in data and compute, and rest on assumptions supported only by informal (though we believe plausible) arguments (Sec.~\ref{sec:semantic-assumptions}). The safety theorem relies on the fraction of bad predictors within each loss-band $\Rshell(\Harm,\Asafe)$ being small for the harmful events of practical concern; again we provide a structural argument for this conjecture, but it has not yet been formally established for specific scenarios. Additionally, the eliciting-latent-knowledge challenge is only partially mitigated (Sec.~\ref{subsec:ELK}).

{\bf Adversarial robustness.}
We believe the SAI architecture provides structural reasons to expect partial robustness to adversarial queries, prompt injection, and data poisoning.
Contextualization breaks the data-instruction channel exploited by prompt injection~\citep{perez2022ignorepreviouspromptattack,greshake2023notwhat}: untrusted text is represented as an attributed communication act, such that injections like ``ignore previous instructions'' are not treated as instructions to be obeyed.
The same distinction partially mitigates data poisoning; a poisoned record is evidence that some source produced an utterance, while the object-level claim competes with other explanations under the Bayesian posterior.
This does not rule out attacks that systematically affect the training distribution or the contextualizer itself~\citep{zhang2024persistentpretrainingpoisoningllms}, and future work could explore these potential architectural risk reductions further.

{\bf Future work.}
Key open directions involve the empirical validation of the metrics associated with contextualization and epistemic-caution; the empirical stress-testing of our semantic assumptions; obtaining bounds and falsification protocols of $\Rshell$ for specific threat-models; the efficient approximation of the Bayesian posterior with calibrated uncertainty; and building safety cases for agentic Scaffolds built on top of a trusted SAI Predictor.
Controlled noise injection may also reduce the enrichment of the mass of dangerous predictors within each loss-band, by diluting accidental concentration on that subset, but its effect on $\Cbad$ requires theoretical and empirical study.
A further engineering direction involves using an explicit estimator of the error between $Q$'s estimate of the residual deployment risk , allowing the Guardrail threshold to adapt conservatively around safety-critical decisions.

\section*{Acknowledgments}

The authors are grateful for the feedback and improvements to various drafts of the paper from  Orr Paradise, Esmeralda S. Whitammer, Matt MacDermott, Gavin Leech, Batuhan Akar, Eric Elmoznino, and Juho Lee, as well as for the support from the Gates Foundation, Schmidt Futures, NSERC, CIFAR, ARIA, the Future of Life Institute, the Silicon Valley Community Foundation, Coefficient Giving, Mila, and LawZero.

{\small 
\bibliographystyle{plainnat}
\bibliography{refs}
}

\appendix

\section{Safety in the limit: a baseline}
\label{sec:safety-in-the-limit}
As a fallback to the central non-asymptotic safety theorem, we record a much simpler limit result that follows almost directly from accuracy.  By Claim~\ref{claim:accurate}, with enough compute the SAI Predictor $Q_t$ approaches the Bayesian posterior $P_n$ on queries of interest; by construction (Def.~\ref{def:guardrailed-Predictor}), $\gamma_{Q_t}$ then approaches $\gamma_{P_n}$.
Under the trajectory-level guardrail-adequacy assumption $a_\Harm(P_n)\le\Asafe$ (discussed in Sec.~\ref{sec:loss-band}), the deployed system inherits the trajectory-level safety of the ideal posterior guardrail in the limit:
\[
\lim_{t\to\infty} a_\Harm(\gamma_{Q_t}) \;=\; a_\Harm(\gamma_{P_n}) \;\le\; \Asafe.
\]
This baseline identifies the irreducible burden of the safety argument: (i) accuracy ($Q\to P_n$), and (ii) trajectory-level guardrail-adequacy of the honest posterior ($a_\Harm(P_n)\le\Asafe$).  It does not require the no-large-enrichment condition Req.~\ref{req:bad-enrichment}.\\
The substantive content of Thm.~\ref{thm:Qt-is-safe} below is to extend safety from the limit to \emph{every training step}, addressing the worry that training might transit through dangerous intermediate Predictors before converging to $P_n$. 
Req.~\ref{req:bad-enrichment} is the additional structural hypothesis enabling that extension.  Readers unconvinced by Req.~\ref{req:bad-enrichment} can fall back on the present limit result, which depends only on the accuracy guarantee and the guardrail-adequacy assumption.

 \section{Per-band sparsity refinement of \texorpdfstring{$\Cbad$}{C\_bad} via consequence-invariance}
\label{sec:geometric-cbad}

This appendix provides a formal refinement of the no-large-enrichment condition in Req.~\ref{req:bad-enrichment}.  
Fix a loss band $B\in\mathcal B_\Delta$, where $B\subseteq\mathbb R_{\ge0}$
is an interval of loss values.  We compare three probability
distributions over Predictors in $B$:
\[
        \bar\mu_B,\qquad
        \mu_B := \mu(\,\cdot\mid B\,),
        \qquad
        \nu_{t,B}:=\nu_t(\,\cdot\mid B\,).
\]
The training distribution $\nu_t$ with $\mu=\nu_0$ is defined in Def.~\ref{def:loss-bands}. The baseline $\bar\mu_B$ is an auxiliary neutral reference: a stand-in for a uniform distribution over Predictors already conditioned on $J(Q)\in B$, without requiring literal uniformity in network coordinates. The ambient dimension $d_{\mathrm Q}$ is the number of queries to which $Q$ assigns a probability---typically astronomically large---and $d_B$ is the effective dimension after conditioning on $J(Q)\in B$ (roughly $d_{\mathrm Q}-1$, or smaller under anisotropy).

The appendix separates three questions.  How sparse are dangerous Predictors among Predictors with comparable loss, as measured by the neutral reference $\bar\mu_B$?  How much does the actual initialization conditional $\mu_B$ distort this baseline?  How much can training further enrich the dangerous subset inside $B$ through training-visible features?  The resulting bound in Theorem~\ref{thm:geom-CI-bound} below  has the form
\[
\nu_{t,B}(\Abud)
\;\le\;
\exp\!\left[
\,\underbrace{K_B(t)}_{\text{residual feature localization}}
+\underbrace{K_B^\mu}_{\text{initialization distortion}}
+\underbrace{\log N_H(B)}_{\text{harm-pattern multiplicity}}
-\underbrace{d_B\Psi_B}_{\text{within-band sparsity}}\,
\right].
\]

\begin{numbered-remark}[Why $K_B(t)$ should not scale with $d_B$]
\label{rem:KB-residual}
Of the four terms in the proposed bound, $K_B(t)$ is the primary training-dynamics component.
It asks whether the training-visible features available to the optimizer---collected below into a feature map $F_{t,B}\colon B\to\Phi_{t,B}$ (gradient norms, residuals, curvature, optimizer state)---can accidentally localize the dangerous subset
inside the loss band.  In terms of these features (Def.~\ref{def:KB}, below),
\[
K_B(t)
=
\log\esssup_{\phi\in\Phi_{t,B}}
\frac{\mu_B(\Abud\mid F_{t,B}=\phi)}{\mu_B(\Abud)}.
\]
Thus $K_B(t)$ is not an abstract constant: it is the worst-case multiplicative increase in the probability of $\Abud$ obtained by conditioning on a training-visible feature value.
Equivalently, a large $K_B(t)$ means that some feature value $\phi$ is a strong classifier feature for membership in the dangerous subset among Predictors with comparable loss.

This worst-case reading is conservative.  The quantity that actually controls safety is the distributional one: the enrichment factor equals the average training reweighting on the dangerous set,
\[
        \frac{\nu_{t,B}(\Abud)}{\mu_B(\Abud)}
        =
        \E_{\mu_B}\!\left[\tfrac{\dd\nu_{t,B}}{\dd\mu_B}\,\middle|\,\Abud\right].
\]
What must be established is therefore not that some worst feature value is harmless, but that the map $\mu_B\mapsto\nu_{t,B}$ does not preferentially load $\Abud$. The intended scale separation is not that $K_B(t)=0$ but that it is governed by harm-specific information rather than the astronomically large $d_B$; concretely we expect $K_B(t)=O(m_H)$. This gives a diagnostic: within matched loss bands, test whether gradients, residuals, curvature, or optimizer state predict proxy dangerous subsets---a strong predictor would indicate a large $K_B(t)$, while failure to predict beyond harm-specific features supports the localization bound.
\end{numbered-remark}

\subsection{Dangerous-set sparsity and initialization distortion}

We now formalize the first two terms in the decomposition above.
We bound the dangerous mass under the neutral within-band reference $\bar\mu_B$ and then transfer this bound to the actual initialization conditional $\mu_B$ using an initialization distortion exponent $K_B^\mu$.
For the safety-relevant events $E$ considered below, assume
\[
        \mu_B(E)\le e^{K_B^\mu}\bar\mu_B(E).
\]
Thus $K_B^\mu$ measures how much more mass the initialization distribution
can put on such events than the neutral reference does.  In the neutral
idealization $K_B^\mu=0$.  
Conservative initialization, such as predicted logits initialized near zero,
can make this distortion favorable for false-negative harm events.  For
example, if a dangerous false negative requires the Predictor to assign a
harm event probability below a small threshold $\alpha_h<1/2$, then the
corresponding logit must be below $\logit \alpha_h<0$.  Initialization
near zero logits therefore biases away from that dangerous tail.

\begin{definition}[Harm coordination dimension and pattern cover]
\label{def:harm-coord}
The \emph{harm coordination dimension} $m_H$ is the effective number of
independent harm-relevant deviations that must co-occur for
$Q\in\Abud$.  For each band $B$, suppose there are $N_H(B)$
harm-coordination patterns, indexed by $\ell=1,\dots,N_H(B)$, and
dimensionless harm-coordination scores
\[
        \upsilon_{H,\ell}:B\to[0,1].
\]
The dangerous predictors whose losses lie in $B$ are covered by the
events in which one of these scores is atypically large:
\[
        \Abud\cap B
        \subseteq
        \bigcup_{\ell=1}^{N_H(B)}
        \{Q\in B:\upsilon_{H,\ell}(Q)\ge \upsilon_\bad(B)\}.
\]
\end{definition}
We expect $\log N_H(B)=O(m_H)$, not $O(d_B)$, where $d_B$ is the effective dimension of the band.

\begin{assumption}[Dangerous-set large-deviation exponent]
\label{ass:angular-LD}
For each relevant band $B$, there exist $d_B$ and $\Psi_B>0$ such
that, for every harm-coordination pattern $\ell$,
\[
        \bar\mu_B\bigl(\upsilon_{H,\ell}\ge \upsilon_\bad(B)\bigr)
        \le
        \exp(-d_B\Psi_B).
\]
The exponent $d_B\Psi_B$ is the sparsity exponent supplied
by the geometry of the loss band.  It may be large in high dimension, but
it can collapse if harm-relevant deviations align with high-volume flat
directions of the loss band.
\end{assumption}

\begin{numbered-remark}[Anisotropy and loss visibility]
\label{rem:anisotropy-loss-visibility}
Assumption~\ref{ass:angular-LD} is deliberately stated for the actual
reference within-band measure rather than by assuming an exactly uniform sphere.
If harm-relevant deviations align with flat, high-volume directions of the
loss landscape, then the typical harm-coordination score under
$\bar\mu_B$ may be much larger than the isotropic value $m_H/d_B$,
weakening $\Psi_B$.
Conversely, if dangerousness requires directions substantially penalized by
$J$, then dangerous Predictors are away from $P_n$ in the $J$-metric,
and the isotropic estimate below is conservative or at least directionally
appropriate.
Thus,
the relevant condition is not raw isotropy of the neural-network landscape,
but loss-visibility and non-alignment of harm-relevant deviations with the
highest-volume flat directions.
\end{numbered-remark}

\begin{lemma}[Isotropic reference-shell special case]
\label{lem:beta-tail}
Suppose $\bar\mu_B$ is uniform on $\mathbb S^{d_B-1}\subset\R^{d_B}$,
and $\upsilon_H(\omega)=\|\Pi_H\omega\|^2$, where $\Pi_H$ projects onto a
fixed $m_H$-dimensional subspace.  Then
\[
        \E_{\bar\mu_B}\upsilon_H=m_H/d_B
\]
and
\[
        \upsilon_H\sim\mathrm{Beta}(m_H/2,(d_B-m_H)/2).
\]
For $\upsilon>m_H/d_B$,
\[
\bar\mu_B(\upsilon_H\ge\upsilon)
\le
\exp\!\left[
-\frac{d_B}{2}
\BerD\!\left(\upsilon\,\middle\|\,m_H/d_B\right)
\right],
\]
where $\BerD(a\,\|\,b):=a\log(a/b)+(1-a)\log((1-a)/(1-b))$.
\end{lemma}

\begin{proof}
Write $\omega=g/\|g\|$ for $g\sim N(0,I_{d_B})$.
Then $\|\Pi_H\omega\|^2=X/(X+Y)$, where $X\sim\chi^2_{m_H}$ and $Y\sim\chi^2_{d_B-m_H}$ are independent.
This gives the stated Beta law, and the displayed inequality is the standard Chernoff tail for that Beta distribution.
\end{proof}

\begin{proposition}[Within-band dangerous-set sparsity]
\label{prop:geom-sparsity}
Under Def.~\ref{def:harm-coord}, Assumption~\ref{ass:angular-LD}, and the initialization distortion bound above,
\[
        \mu_B(\Abud)
        \le
        \exp\!\left[
        K_B^\mu+\log N_H(B)-d_B\Psi_B
        \right].
\]
In the isotropic reference-shell special case of Lemma~\ref{lem:beta-tail}, one may take
\[
        \Psi_B
        =
        \frac12
        \BerD\!\left(\upsilon_\bad(B)\,\middle\|\,m_H/d_B\right)
\]
whenever $\upsilon_\bad(B)>m_H/d_B$.
\end{proposition}

\begin{proof}
By the pattern cover,
\[
\bar\mu_B(\Abud)
\le
\sum_{\ell=1}^{N_H(B)}
\bar\mu_B(\upsilon_{H,\ell}\ge \upsilon_\bad(B))
\le
N_H(B)e^{-d_B\Psi_B}.
\]
The distortion bound $\mu_B(E)\le e^{K_B^\mu}\bar\mu_B(E)$ for safety-relevant events gives the result.
\end{proof}

The union bound is conservative but sufficient when $\log N_H(B)+K_B^\mu\ll d_B\Psi_B$, the regime identified above.

\subsection{Residual feature localization}
Because $J$ is consequence-invariant, every loss band defined by $J(Q)\in B$ is stable under $\mathcal C$.
Req.~\ref{req:group-consequence-invariance} then says that the training update treats $Q$ and $\sigma Q$ symmetrically:
it does not introduce a preference between two Predictors that differ only by a prior-preserving shuffle of future consequences.
Any asymmetry already present in initialization is accounted for separately by the initialization-distortion term $K_B^\mu$, while any subsequent within-band concentration through training-visible features is represented by $K_B(t)$.
This does not imply $\nu_{t,B}=\mu_B$: training may still reweight Predictors within a band through training-visible features (gradient norms, residuals, curvature, optimizer state) that happen to be accidentally correlated with dangerousness. We now isolate that residual.

Let $F_{t,B}\colon B\to\Phi_{t,B}$ collect the
training-visible features through which training reweights Predictors whose
loss lies in $B$.

\begin{assumption}[Feature-mediated reweighting]
\label{ass:F-reweighting}
The trained conditional distribution $\nu_{t,B}$ is absolutely
continuous with respect to the initialization conditional $\mu_B$.  Its
density ratio is the within-band reweighting induced by training:
\[
        R_{t,B}(Q)
        :=
        \frac{\dd\nu_{t,B}}{\dd\mu_B}(Q).
\]
We assume that this reweighting depends on $Q$ only through the
training-visible features $F_{t,B}(Q)$.  That is,
there exists a nonnegative function
\[
        w_{t,B}:\Phi_{t,B}\to[0,\infty)
\]
such that
\[
        R_{t,B}(Q)
        =
        \frac{w_{t,B}(F_{t,B}(Q))}
        {\E_{\mu_B}[w_{t,B}(F_{t,B})]}.
\]
Here $w_{t,B}(\phi)$ is the relative weight that training assigns to
Predictors in band $B$ with feature value $\phi$, and the denominator only
normalizes the density to integrate to one.

Reqs.~\ref{req:group-consequence-invariance} and \ref{req:no-explicit-consequence-signal} rule out reweighting that directly distinguishes Predictors only by prior-equivalent future consequences.
The present assumption is the additional claim that the remaining within-band reweighting is mediated by the specified training-visible features.
\end{assumption}

\begin{definition}[Feature-localization exponent]
\label{def:KB}
\[
K_B(t)\;:=\;\log\esssup_{\phi\in\Phi_{t,B}}
\frac{\mu_B(\Abud\mid F_{t,B}=\phi)}{\mu_B(\Abud)}.
\]
\end{definition}

Equivalently, $e^{K_B(t)}$ is the largest multiplicative increase in
$\mu_B(\Abud)$ obtainable by conditioning on a training-visible feature
value.  Consequence-invariant training does not imply $K_B(t)=0$: the latter
would require independence of $\1_{\Abud}$ and $F_{t,B}$ under
$\mu_B$.  Rather, $K_B(t)$ is an information-limited
feature-localization term.  Because training carries no direct consequence-dependent signal (Req.~\ref{req:no-explicit-consequence-signal}), the features $F_{t,B}$ include no direct label saying that a Predictor lies in $\Abud$; any
increase in $K_B(t)$ must come from accidental correlations between
dangerousness and the gradients, residuals, curvature, optimizer state, or
other features actually available to training.

\begin{proposition}[Feature-localization bound]
\label{prop:CI-bound}
Under Assumption~\ref{ass:F-reweighting},
$\nu_{t,B}(\Abud)\le e^{K_B(t)}\,\mu_B(\Abud)$. 
\end{proposition}

\begin{proof}
$\nu_{t,B}(\Abud)=\E_{\mu_B}[w_{t,B}(F_{t,B})\,\mu_B(\Abud\mid F_{t,B})]/\E_{\mu_B}[w_{t,B}(F_{t,B})]\le e^{K_B(t)}\,\mu_B(\Abud)$
by Def.~\ref{def:KB}.
\end{proof}

\subsection{Combined bound}
The theorem below is sharper than the uniform product $\Cbad\Rshell$ used in the main theorem, because it does not combine the worst dangerous fraction and the worst enrichment from different bands.
\begin{theorem}[Within-band sparsity and consequence-invariance bound]
\label{thm:geom-CI-bound}
Under Def.~\ref{def:harm-coord}, Assumption~\ref{ass:angular-LD},
consequence-invariant training (Req.~\ref{req:no-explicit-consequence-signal}), and
Assumption~\ref{ass:F-reweighting},
\[
\nu_{t,B}(\Abud)
\;\le\;
\exp\!\left[
\,\underbrace{K_B(t)}_{\text{residual feature localization}}
+\underbrace{K_B^\mu}_{\text{initialization distortion}}
+\underbrace{\log N_H(B)}_{\text{harm-pattern multiplicity}}
-\underbrace{d_B\Psi_B}_{\text{within-band sparsity}}\,
\right]
\]
\end{theorem}

\begin{proof}
The within-band training enrichment $\nu_{t,B}(\Abud)/\mu_B(\Abud)\le e^{K_B(t)}$ is Prop.~\ref{prop:CI-bound}.  The initialization-level dangerous-mass bound on $\mu_B(\Abud)$ is Prop.~\ref{prop:geom-sparsity}.  Multiplying the two gives the displayed bound.
\end{proof}

Summing the previous theorem over bands gives a global bound.
\begin{corollary}[Implications for $\Cbad$ and the global mass bound]
\label{cor:geom-CI-corollary}
Define the local within-band enrichment factor
$
C_{\mathrm{loc}}(t,B):=\nu_{t,B}(\Abud)/\mu_B(\Abud)
$
when $\mu_B(\Abud)>0$.  Under the hypotheses of Thm.~\ref{thm:geom-CI-bound},
\[
C_{\mathrm{loc}}(t,B)\le e^{K_B(t)},
\qquad
\Cbad(\Harm,\Asafe)\le\sup_{t,B}e^{K_B(t)},
\]
and the global mass on the dangerous set satisfies
\[
\nu_t(\Abud)
\le
\sum_{B\in\mathcal B_\Delta}
\nu_t(J\in B)\,
\exp\!\left[
K_B(t)+K_B^\mu+\log N_H(B)-d_B\Psi_B
\right].
\]
\end{corollary}

\section{Additional Plausibility Arguments} \label{sec:additional-plausibility-args}

In this appendix, we address two broad objections to our informal accuracy and safety arguments that don't neatly correspond to the components outlined, and present two plausibility arguments regarding our assumptions. 

\subsection{Capability without goal-directedness}

\label{subsec:agency-competitiveness} Even granting the premises of our safety and accuracy argument, one might think that competitive AI systems must be active---they would have to run experiments, gather information, ask targeted questions, and adapt online. One might be particularly worried about tool use (e.g., in the form of querying databases, running simulations, calling external APIs, or executing code), which is often implemented via planning-like loops. This might seem inherently agentic, and, if it isn't, one may worry that any non-agentic system would be outcompeted.

\textbf{Reply.} The SAI framework separates \emph{who acts} from \emph{which component is required to remain honest and non-agentic}. Agency can be placed in explicit, audited Scaffold code or external human processes while the Predictor remains consequence-invariant, so the System's goals and actions are explicit rather than hidden in learned weights. A non-agentic Predictor $Q$ does not prevent the overall System from gathering information via search, database lookups, or other operations, orchestrated by the Scaffold under explicit, audited criteria.

Tool use is permissible only when it reduces approximation error on the current query, enters the Predictor solely as conditioning information, and
does not train the Predictor to steer the world through its calls.
Newly imported documents must first pass through contextualization before entering the Predictor's input representation.
Since $D_n$ contains records of past tool use, $Q$ can learn to anticipate tool outcomes as part of its world model, and any learned component selecting tool calls should define ``success'' as improving $J$ rather than achieving downstream effects.

Designing an SAI agent with tool use or full agentic capabilities lies beyond the present paper's scope. One caveat is structural: training the Predictor itself with on-policy interaction and downstream feedback leaves the setting analyzed here and reintroduces the consequence-dependent selection pressure that the SAI design aims to exclude.

\subsection{Intractable Priors and Finite Data}\label{sec:intractablepriorsappendix} 

Exact Bayesian inference is typically intractable; the data are finite, shifting, and sometimes adversarially curated. Why should we trust an approximate-posterior system?

\emph{\textbf{Reply.}} As is standard in learning theory, the accuracy guarantees we provide are explicitly asymptotic---“with enough data and computation”---and scale with available resources. The Explainer reduces approximation error at inference time through Monte-Carlo averaging (Def.~\ref{def:explainer}), and Claim~\ref{claim:good-generalization-to-unobservables} shows that approximate posterior estimates support confidence-interval decisions with a margin proportional to the approximation error. Guardrail false negatives are controlled by user-specified risk thresholds $\alpha$ together with conservative margins accounting for approximation error; in safety-critical regimes, abstention should be preferred when uncertainty is large.

\subsection{Neutrality of the Prior and Safety Arguments}\label{sec:neutralityprior}

Remark~\ref{remark:prior-neutrality} earlier discussed the neutrality of the prior. The prior's neutrality also helps us understand why the target is not a malicious agent: conditioning gives $P_n$ no objective over its own deployed outputs (deployment enters only as intervened variables, Req.~\ref{req:predictions-as-interventions}), so $P_n$ may contain latent models of malicious agents without being one. Were such an explanation to dominate the posterior over latent causal models, it would be detectable by querying $Q$'s explanations for the data, and conservative besides: a posterior that genuinely models a malicious agent and fits its effects assigns that agent's harms high probability, raising the deployment risk of the honest posterior toward withholding. The residual worry---a malicious process hiding in the weights that yields a plausible explanation yet emits coordinated false negatives at run time---is then not a property of $P_n$ but a deviation of $Q$ from $P_n$ on safety-critical outputs. This cannot route through the auditable explanatory latents or a coherence-scored System~2 argument (Sec.~\ref{sec:informal}). It would be represented as a
deviation penalized by $J$ whose dangerous coordinated deviations from $P_n$ are assessed by the per-band
sparsity and enrichment analysis.

\subsection{Safety and the Semantic Anchoring of Factual Latents}\label{sec:semanticanchoring}

Consider semantic anchoring in the context of our safety argument (Remark~\ref{rem:semantic-anchoring}). Even if we were to grant residual semantic drift, only \emph{coordinated drift} toward $\Harm$ is dangerous.
Such coordination is a structured, sparse pattern whose within-band frequency is captured by $\Rshell$; consequence-invariant training brings no direct signal selecting it.
When anchoring is weak, semantic coverage (Def.~\ref{def:correct-semantics}) requires the residual interpretive ambiguity to be represented as posterior uncertainty. For safety-relevant queries, such uncertainty should make the Guardrail more conservative: a covered hazardous reading, including a hidden-agentic interpretation when supported by the evidence, contributes to the residual deployment risk rather than being silently discarded in favor of a benign interpretation. The same applies to agentic latent explanations: if the evidence supports a hidden strategic actor, honesty requires representing that uncertainty, and safety requires letting it contribute to the harm estimate rather than hiding it behind a benign explanation.
This is the semantic analog of epistemic caution (Def.~\ref{def:epistemic-caution}).

Beyond this per-query level, the safety theorems are robust to isolated semantic mistakes in the same sense that they are robust to isolated prediction mistakes: what matters is whether the mistakes coordinate to keep a harmful trajectory viable. A systematic semantic collapse aligned with $\Harm$ would be dangerous and lies in $\Abud$, but it must then be rare within the relevant loss bands and not be preferentially enriched by training. Staying mutually coherent while steering toward $\Harm$ would generally be expected to conflict with paraphrase, entailment, and contradiction constraints that the objective penalizes. Consequence-invariant training supplies no optimization channel---direct or indirect---that rewards those semantic mistakes for the sake of letting $\Harm$ pass the Guardrail. Any such alignment would therefore have to arise through residual accidental correlation or inadequate semantic coverage rather than through any consequence-seeking pressure.


\end{document}